\documentclass[nohyperref]{article}

\usepackage{graphicx}

\usepackage{hyperref}

\usepackage{algorithm}
\usepackage[noend]{algorithmic}

\usepackage[accepted]{icml2022}

\usepackage{caption} 

\usepackage{graphicx} 
\usepackage{natbib}  
\usepackage{textcomp}

\usepackage{import} 
\usepackage{multirow} 
\usepackage{minipage-marginpar}
\usepackage{amsmath,amsfonts,amsthm}

\icmltitlerunning{Order Constraints in Optimal Transport }

\begin{document}

\twocolumn[
\icmltitle{Order Constraints in Optimal Transport }



\icmlsetsymbol{equal}{*}

\begin{icmlauthorlist}
\icmlauthor{Fabian Lim}{ibm}
\icmlauthor{Laura Wynter}{ibm}
\icmlauthor{Shiau Hong Lim}{ibm}
\end{icmlauthorlist}

\icmlaffiliation{ibm}{IBM Research, Singapore}

\icmlcorrespondingauthor{Fabian Lim}{flim@sg.ibm.com}

\icmlkeywords{Machine Learning, ICML}

\vskip 0.3in
]



\printAffiliationsAndNotice{}  




\theoremstyle{plain}
\newtheorem{theorem}{Theorem}[section]
\newtheorem{proposition}[theorem]{Proposition}
\newtheorem{lemma}[theorem]{Lemma}
\newtheorem{corollary}[theorem]{Corollary}
\theoremstyle{definition}
\newtheorem{definition}[theorem]{Definition}
\newtheorem{assumption}[theorem]{Assumption}
\theoremstyle{remark}
\newtheorem{remark}[theorem]{Remark}


\DeclareRobustCommand{\comment}[1]{\textcolor{red}{#1}}

\newcommand{\etal}{\textit{et.~al.}}
\newcommand{\etc}{\textit{etc.}}
\newcommand{\ie}{\textit{i.e.}}
\newcommand{\eg}{\textit{e.g.}}

\newcommand{\tplan}{\Pi}
\newcommand{\tpoly}{U(a,b)}
\newcommand{\cost}{\mathcal{C}}

\newcommand{\ones}{\mathbf{1}}
\newcommand{\mOnes}[1]{\mathbf{1}_{#1}}
\newcommand{\mIdentity}[1]{\mathbf{I}_{#1}}
\newcommand{\mZeros}{\mathbf{0}}

\newcommand{\trace}[1]{\mbox{\textup{tr}} \left(#1\right)}
\newcommand{\rank}[1]{\mbox{\textup{rank}} \left(#1\right)}
\newcommand{\bigO}[1]{\mathcal{O} \left(#1\right)}
\newcommand{\Real}{\mathbb{R}}
\newcommand{\mnReal}{\Real^{m \times n}}
\newcommand{\mnRealp}{\Real^{m \times n}_+}

\newcommand{\iset}[1]{\left[#1\right]} 
\newcommand{\pset}[1]{2^{\iset{#1}}}
\newcommand{\idx}[2]{{#1}_{#2}} 
\renewcommand{\ij}[2]{#1#2} 
\newcommand{\ijset}[2]{\iset{\ij{#1}{#2}}} 
\newcommand{\ijdx}[3]{\idx{#1}{\ij{#2}{#3}}} 
\newcommand{\mij}[3]{
    \ij{#1_1}{#2_1},\ij{#1_2}{#2_2},\cdots,\ij{#1_{#3}}{#2_{#3}}
} 
\newcommand{\mijdx}[4]{{#1}_{\mij{#2}{#3}{#4}}}
\newcommand{\nij}[3]{
    \ij{#1}{{#2}_{\iset{#3}}}
}

\renewcommand{\ni}[2]{
    {#1}_{\iset{#2}}
}
\newcommand{\nijdx}[4]{{#1}_{\nij{#2}{#3}{#4}}}

\newcommand{\odidx}[2]{{#1}_{(#2)}} 
\newcommand{\tidx}[2]{{#1}_{#2}} 

\newcommand{\basepoly}{\mathcal{B}_F}  

\newcommand{\restrict}[2]{\sum_{i \in #2} {#1}_i}

\newcommand{\saturation}{\phi}
\newcommand{\heuristic}{\Phi}
\newcommand{\selfheu}{\saturation^s}
\newcommand{\rowheu}{\saturation^r}
\newcommand{\colheu}{\saturation^c}
\newcommand{\variates}{\mathcal{I}}
\newcommand{\stack}{\mathcal{S}}

\newcommand{\order}{O}

\newcommand{\proj}[2]{\mbox{Proj}_{#1} \left( #2 \right)}
\newcommand{\modulo}{\mbox{mod}}

\newcommand{\norm}[2]{\left\Vert #1 \right\Vert_{#2}}
\newcommand{\frobenius}[1]{\norm{#1}{F}}

\newcommand{\intersect}{\cap}
\newcommand{\union}{\cup}
\newcommand{\setsym}{\mathcal{C}} 

\newcommand{\constraint}[1]{\setsym_{#1}}
\newcommand{\constraintOne}{\constraint{2}}
\newcommand{\constraintTwo}{\constraint{2}}
\newcommand{\constraintThree}{\constraint{1}(a,b)}
\newcommand{\constraintAll}{
     \constraintThree \cap \constraintOne 
}

\newcommand{\meanproj}[1]{\tidx{P}{#1}}

\newcommand{\zeroprog}[1]{\left( #1 \right)_+}

\newcommand{\langidx}{V}
\newcommand{\psum}{\psi}
\newcommand{\pmean}{\Delta}

\newcommand{\searchtree}{\mathcal{T}}
\newcommand{\esttree}{\widehat{\mathcal{T}}}

\newcommand{\indicate}[2]{\mathbb{I}_{#1} ( #2 )}

\newcommand{\bl}{B}

\newcommand{\submodularub}{G}
\newcommand{\submodularlb}{\varphi}
\newcommand{\boundcoef}{\submodularlb}
\newcommand{\tailfun}{g}
\newcommand{\knapsack}[1]{\mbox{\texttt{PACKING}} \left( #1 \right)}
\newcommand{\knapmu}{\mu}
\newcommand{\knapnu}{\nu}

\newcommand{\tailfunlbRow}[2]{{\tidx{L}{1}}_{#1}\left(#2\right)}  
\newcommand{\tailfunlbCol}[2]{{\tidx{L}{2}}_{#1}\left(#2\right)}  

\newcommand{\diminish}[2]{
    \mbox{\texttt{diminish}} \left(#1, #2 \right)
}

\newcommand{\floor}[1]{\left\lfloor #1 \right\rfloor}

\newcommand{\oiidx}[2]{{#1}_{[#2]}} 

\newcommand{\valuebound}{L_{\infty}}
\newcommand{\roundAll}[1]{\widehat{#1}}

\newcommand{\Aeq}{A}

\newcommand{\pinv}[1]{{#1}^\dagger}
\newcommand{\inv}[1]{{#1}^{-1}}

\newcommand{\mspan}[1]{
    \mbox{\textup{span}} \left(#1\right)}
\newcommand{\nullspace}[1]{
    \mbox{\textup{null}} \left(#1\right)}

\newcommand{\kron}{\otimes}

\renewcommand{\matrix}[2]{
    \left[
        \begin{array}{#1}
            #2
        \end{array}
    \right]
}
\renewcommand{\vector}[1]{
    \matrix{c}{#1}
}

\newcommand{\ravel}[1]{
    \mbox{
        \texttt{ravel}
    } \left(#1\right)}
\newcommand{\unravel}[1]{
    \mbox{
        \texttt{unravel}
    } \left(#1\right)}



\begin{abstract}

Optimal transport is a framework for comparing measures whereby a cost is incurred for transporting one measure to another.
Recent works have aimed to improve  optimal transport plans through the introduction of various forms of structure. 
We introduce novel order constraints into the  optimal transport formulation to allow for the incorporation of  structure. 
We define an efficient method for  obtaining explainable solutions to the new formulation that scales far better than standard approaches.
The theoretical properties of the method  are provided.
We demonstrate experimentally that order constraints improve explainability  using the e-SNLI (Stanford Natural Language Inference) dataset that includes human-annotated rationales as well as on several image color transfer examples.

\end{abstract}


\section{Introduction}
\label{sec:introduction}

Optimal transport (OT) is a framework for comparing measures whereby a cost is incurred for transporting one measure to another. Optimal transport  enjoys both significant theoretical interest and  widespread applicability \citep{villani, peyre_computational_2020}.
Recent work \citep{swanson_rationalizing_2020,alvarez-melis_structured_2017}  aims to improve the interpretability of optimal transport plans through the introduction of \textit{structure}. Structure can take many forms: text documents where   context is given  \cite{hierarchy}, images with  certain desired features   \cite{imageAAAI2020, imageAAAI2021},   fairness properties \cite{laclau}, or even patterns in RNA \cite{forrow}. \citet{swanson_rationalizing_2020}   postulate that structure  can be discovered by sparsifying optimal transport plans. 

Motivated by advances in sparse model learning, order statistics and isotonic regression \cite{regularity}, we introduce novel order constraints (OC) into the optimal transport formulation so as to allow for more complex structure to be readily added to optimal transport plans.
We show theoretically that, for convex cost functions with efficiently computable gradients,  the resulting order-constrained  optimal transport problem can be solved using a form of ADMM \cite{admm} and can be $\delta$-approximated efficiently.
We further derive computationally efficient lower bounds for our formulation. 
Specifically, when the structure is not provided in advance, we show how it can be estimated. The bounds allow the use of an  explainable   method for identifying the most important constraints, rendered tractable through branch-and-bound.  Each of these  order constraints  leads to an optimal transport plan  computed sequentially using the algorithm we introduce. 
The end result is a diverse set of the most important optimal transport plans from which a user can select the plan that is preferred.

\begin{figure*}
    \centering
    \includegraphics[width=0.75\textwidth]{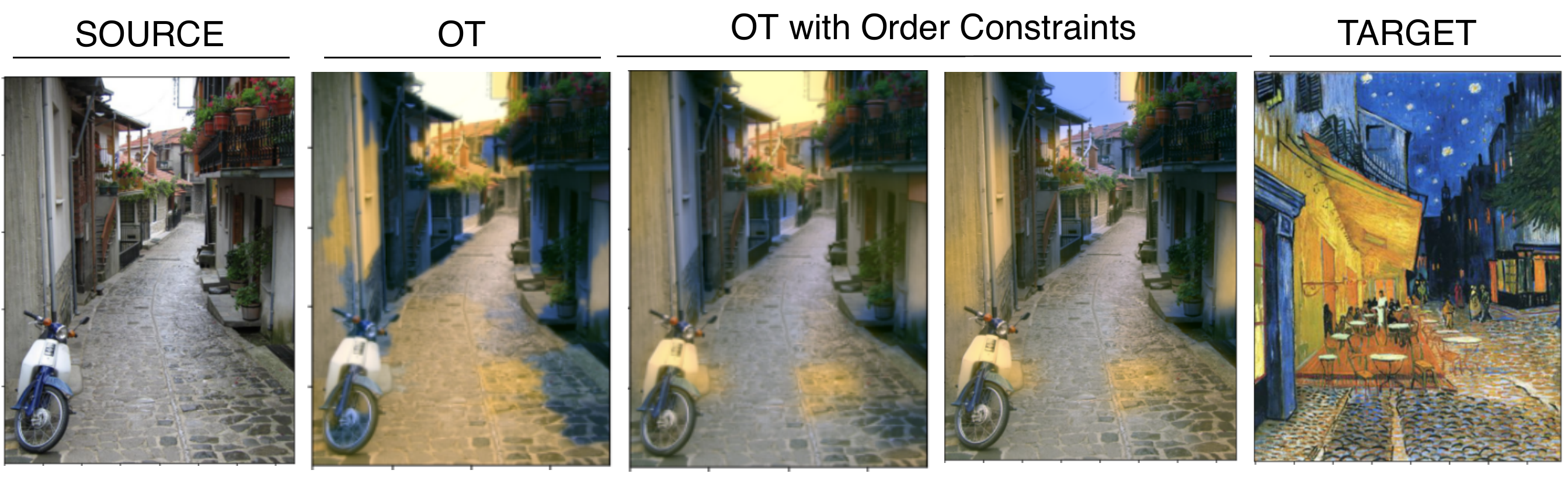}
    \caption{Color transfer using  OT with Order Constraints }
    \label{fig:intro-example-ct}
\end{figure*}

 The order constraints   allow context to be taken into account, explicitly, when known in advance, or implicitly, when estimated through the proposed procedure. 
Consider OT  for document retrieval (\eg,~see~\citet{kusner_word_2015,swanson_rationalizing_2020}) where a user enters a text query.  The multiple transport plans provided by   OT with order constraints offer a diverse set of the best responses, each accounting for different possible contexts.
The resulting method is weakly-supervised with only two tunable thresholds (or unsupervised if the parameters are set to default values), and is  amenable to  applications of OT where labels are not available  to train a supervised model.

\begin{figure}
    \centering
    \includegraphics[width=0.95\columnwidth]{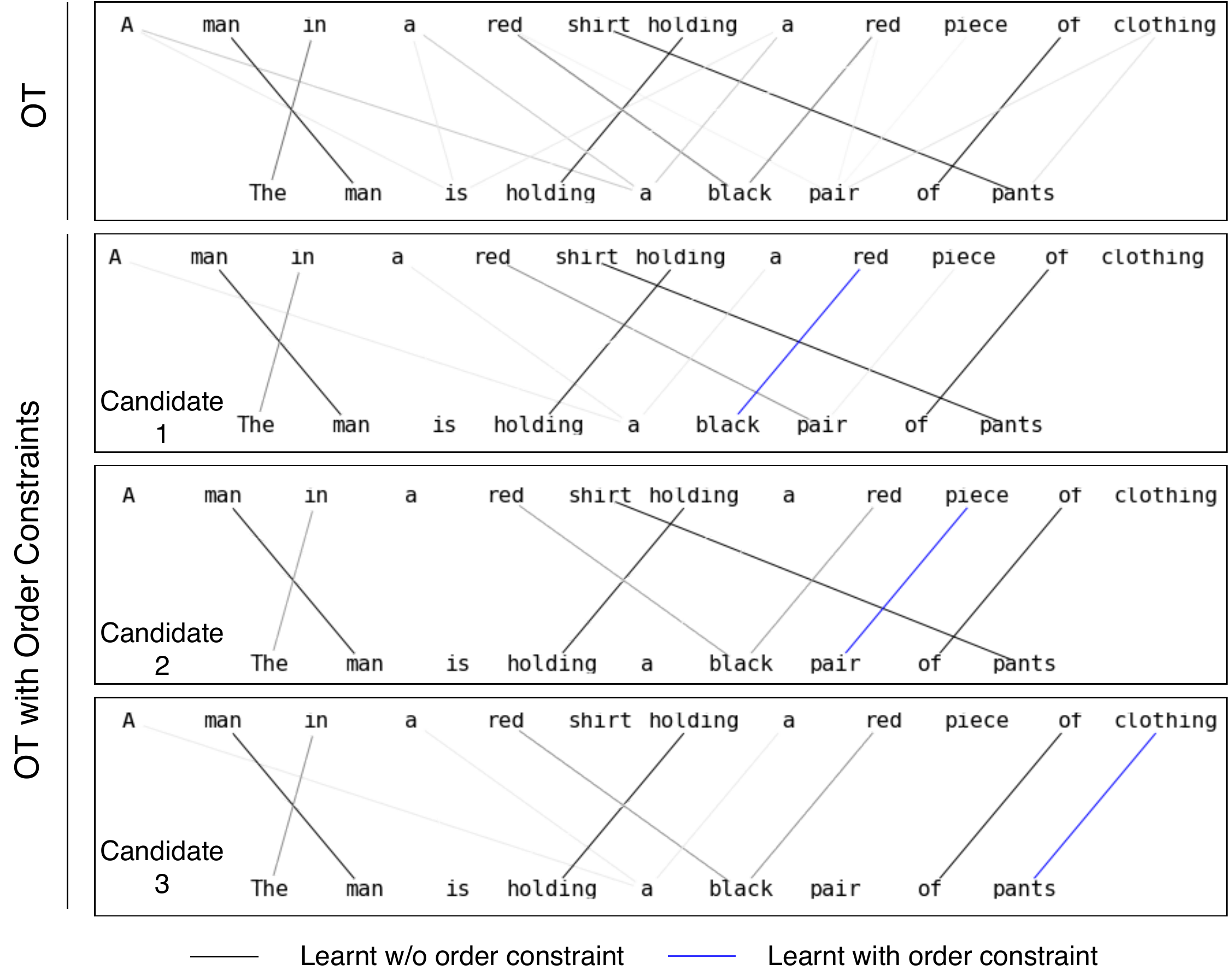}
    \caption{ Text matching using  OT with Order Constraints}
    \label{fig:intro-example}
\end{figure}

An example is provided in Fig. \ref{fig:intro-example} which aims to learn the  similarity between two phrases. The top box depicts the transport plan learnt by  OT; the darker the line, the stronger the transport measure. The match provided by OT does not  reveal the contradiction in the two phrases.
The second to fourth transport plans are more human-interpretable. 
The second plan  matches  ``red'' clothing to the  ``black'' pants. 
The third  matches  ``piece'' to ``pair'', a relation not included in the standard OT plan.
The last plan  matches ``clothing'' to ``pants''.\footnote{For readability of Fig. \ref{fig:intro-example}, very low   scores were suppressed, making it appear that  they do not  sum to one though they do.}
The  example makes use of   OT with a single order constraint to better illustrate the idea of order constraints. OT with multiple order constraints can  incorporate all of the pairs in a single plan.

Fig.~\ref{fig:intro-example-ct} illustrates using OT to transfer the color palette of a target image (Van Gogh's "Café Terrace at Night") to the source image of an alley. 
The second image from the left, obtained by standard OT,  transfers yellow to bright areas and blue to dark. Our proposed order constraints allow  adding context, which in this case  gives
 desired effects such as transferring the Van Gogh night sky (second image from right) or the awning brilliant yellow (middle image) to the sky in the source image.

The contributions of this work are as follows: (1) we provide a novel formulation, Optimal Transport with Order Constraints, that allows  explicitly incorporating structure to OT plans; (2) we provide an efficient method for solving the proposed OT with OC formulation related to an ADMM along with the convergence theory; (3) for cases where the order constraints are not known in advance, we propose an explainable method to estimate them using branch \& bound and define bounds for use in the method, and (4) we demonstrate  the benefits of using the proposed method on NLP and image color transfer tasks.

Next, we review  related work. In the following section, we present the standard  optimal transport formulation followed by  the proposed OT with order constraints (OT with OC) formulation and a demonstration of the existence of a solution, which corresponds to contribution (1)  above. Then, an algorithm leveraging the form of the polytope is  defined to solve the OT with OC  problem in a manner  far more efficient than a standard convex optimization solver, as noted in contribution (2) above. In order to further improve explainability, when the order constraints are not known a priori, we present a branch \& bound procedure to search the space for a set of diverse OT with OC plans, along with  the bounds that reduce the polynomial search space, as noted in contribution (3) above.
Finally,  experiments are provided on the e-SNLI  dataset and several image color transfer examples to demonstrate how order constraints allow for incorporating structure and  improving explainability, as noted in contribution (4) above.

\paragraph{Related Work}

We consider  point-point linear costs as  in~\eg,~\citet{cuturi_sinkhorn_2013,courty_optimal_2016,kusner_word_2015,wu_word_2018, yurochkin_hierarchical_2019, solomon_optimal_2018,altschuler_massively_2019,altschuler_near-linear_2017,schmitzer_stabilized_2019}.
Recently, a number of works have sought to  add structure to  optimal transport  so as to obtain more intuitive and explainable transport plans. 
\citet{alvarez-melis_towards_2019}  modeled invariances in OT,  a form of regularization, to learn the most meaningful transport plan.
Sparsity has been explored in OT to implicitly improve explainablity. 
\citet{swanson_rationalizing_2020} used ``dummy'' variates to    include fewer transport coefficients. \citet{forrow} regularize using the OT rank to better deal with high dimensional data.
\citet{blondel_smooth_2017} had a similar goal of learning sparse and potentially more intepretable transport plans through regularization and a dual formulation.  We, on the other hand, propose an explicit approach to interpretability, rather than the implicit method of sparsifying solutions through regularization.  \cite{DTW} add constraints to require two sequences to lie close to each other, in terms of  indices; to respect sequence order, they penalize the   constraints and then use the classic Sinkhorn algorithm to obtain an approximate solution.
Another way to add structure into OT is through dependency modeling;  multi-marginal optimal transport takes this approach, wherein multiple measures are transported to each other. The formulation was  shown to be generally intractable~\citep{altschuler_hardness_2020}, however, and strong assumptions have to be made to solve it in polynomial time~\citep{altschuler_polynomial-time_2020,pass_multi-marginal_2014,di_marino_optimal_2015}.

Standard OT can be solved in cubic time by primal-dual methods \citep{peyre_computational_2020} and the proposed OT with order constraints can  be solved using generalized solvers with similar time complexity. The difficulty comes from the fact  that there are   quadratically  many order constraints.
Recent work~\citep{cuturi_sinkhorn_2013, scetbon_low-rank_2021, altschuler_near-linear_2017,goldfeld, zhang2021convergence,lin2019a,guo2020,jambulapati_direct_2019,guminov_combination_2021}  has favored  solving  OT via iterative methods which offer lower complexity at the expense of obtaining an approximate solution.
In particular,~\citep{scetbon_low-rank_2021} considered low-rank approximations of OT plans, and proposed a mirror descent with inner Dykstra iterations.  Recent iterative methods for standard OT \cite{guminov_combination_2021, jambulapati_direct_2019} ensure feasibility of the solution at each iteration through the use of a rounding algorithm proposed by \cite{altschuler_near-linear_2017}.
Iterative methods are  essential for solving  submodular OT \citep{alvarez-melis_structured_2017} that would otherwise incur an exponential number of explicit constraints.



\section{ Order Constraints for Optimal Transport}


\paragraph{Notation} Let $\Real$ (resp. $\Real_+$) denote the set of real (resp. non-negative) numbers. 
 Let $m, n$  denote the number of rows and columns in the optimal transport problem.
Vectors and matrices are in $\Real^n$ and $\mnReal$.
Let $\iset{n}:=\{1,2\ldots n\}$ and $\ijset{m}{n}:=\iset{m}\times\iset{n}$.
Let $i,j,k,\ell,p,q$  denote vector/matrix indices,
 $\pset{n}$ denote the power set of $\iset{n}$,
and  $\mOnes{n}$ and $\mIdentity{n}$ denote the all-ones vector and identity matrix, resp., of size $n$.
We denote the index sequence $i_1, i_2, \cdots, i_k$ as $i_{[k]}$. 
$\indicate{S}$ is the indicator function over set $S$.
The \emph{trace} operator for square matrices is denoted by $\trace{\cdot}$ and
$\rank{\ijdx{X}{i}{j}}$ gives the descending positions over the $mn$ matrix coefficients $\ijdx{X}{i}{j}$, or a subset thereof.\footnote{This is not to be confused with the standard definition of a matrix's rank.}
For a \emph{closed set} $\setsym$, the Euclidean projector for a matrix $X \in \mnReal$ is
$
    \proj{\setsym}{X} = \arg \min_{Y \in \setsym} \frobenius{Y-X}^2
$
where $\frobenius{\cdot}$ is the matrix \emph{Frobenius} norm.
For $c \in \Real$ let $\zeroprog{c} = c$ if $c\geq 0$ and $\zeroprog{c} = 0$ otherwise, and 
for $X \in \mnReal$ let $\zeroprog{X}$ equal $X$
after setting all negative coefficients to zero.

\paragraph{Standard OT Formulation}

In optimal transport, one  optimizes a linear transport cost over a simplex-like polytope, $\tpoly$. 
A (balanced) optimal transport problem is given by two  probability vectors $a$ and $b$ that each sum to 1.
Each being of length $m, n$, resp., they define a polytope:
\[
    \tpoly = \{\tplan \in \mnReal_+ : \tplan \mOnes{n} = a, \tplan^T \mOnes{m} = b\},
\]
and a linear cost $D \in \mnReal$, where  $D \geq 0$ and bounded.
Then the optimal transport cost $f:\mnReal\to\Real_+$ is the optimal cost of the following optimization:
\begin{equation}
    \min_{\tplan \in \tpoly} f(\tplan) := 
    \trace{D^T \tplan},
    \label{eqn:optimal-transport}
\end{equation}
and the optimal value of $\tplan$ is the \textit{transport plan}.
We propose to extend \eqref{eqn:optimal-transport} by including   order constraints to represent structure.

\paragraph{ Proposed Formulation: OT with Order Constraints}
The use of order constraints is  analogous to the constraints used in  isotonic, or monotonic, regression,~\citep[see][]{de_leeuw_isotone_2009}.
Specifically, if a particular variate $\ijdx{\tplan}{i}{j}$ is  semantically or symbolically important, then its $\rank{\ijdx{\tplan}{i}{j}}$   should reflect  that dominance.
We thus enforce $k$ order constraints by specifying a sequence of  $k$ variates in $\iset{mn}$, $\nij{i}{j}{k} := \mij{i}{j}{k} $ of importance, where the set of variates $\langidx := \ijset{m}{n} \setminus \{\ij{i_\ell}{j_\ell}: \ell \in \iset{k}\}$:
\begin{equation}
    \ijdx{\tplan}{i_k}{j_k}
    \geq 
    \cdots
    \geq
    \ijdx{\tplan}{i_1}{j_1}
    \geq 
    \ijdx{\tplan}{p}{q}~~\mbox{for}~\ij{p}{q} \in \langidx.
    \label{eqn:transport-plan-ordering-constraints}
\end{equation}

Enforcing order constraints  \eqref{eqn:transport-plan-ordering-constraints} means each variate $\ij{i_\ell}{j_\ell}$ is fixed to maintain the $(k-\ell+1)$-th top position $\odidx{\tplan}{k-\ell+1}$ in the ordering
$\odidx{\tplan}{1} \geq \odidx{\tplan}{2} \geq \cdots \geq \odidx{\tplan}{mn}$
whilst learning the transport plan $\tplan$; in particular $\ij{i_k}{j_k}$ is fixed at the topmost position.
Therefore, for order constraints on $k$ variates of importance and point-to-point costs $D \in \mnReal$,  we extend the OT problem    as follows:
\begin{eqnarray}
&\inf_{\tplan \in \tpoly}   f(\tplan) :=  \trace{D^T \tplan}, \label{eqn:obj2} \\
  \mbox{s.t}\; &\ijdx{\tplan}{i_k}{j_k} \geq \cdots \geq  \ijdx{\tplan}{i_1}{j_1},~~\ijdx{\tplan}{i_1}{j_1}  
  \geq \ijdx{\tplan}{p}{q}~~\mbox{for}~\ij{p}{q} \in \langidx,
    \label{eqn:transport-plan-ordering-constraints-3}
\end{eqnarray}
for $\langidx$ in~\eqref{eqn:transport-plan-ordering-constraints}, and we assume non-negative costs $D \geq 0 $ throughout. 
Let $\nijdx{\order}{i}{j}{k} \subseteq \mnRealp=$ 
\begin{eqnarray}
        \{
            X \in \mnRealp:~\rank{\ijdx{X}{i_\ell}{j_\ell}} = k-\ell+1,~\forall \ell \in \iset{k}
        \}.  \label{eqn:multi-order} 
\end{eqnarray} 
Thus  \eqref{eqn:obj2}--\eqref{eqn:transport-plan-ordering-constraints-3} can be compactly expressed as:
\begin{equation}
    \inf_{
        \tplan \in \tpoly \intersect \nijdx{\order}{i}{j}{k}
    } f(\tplan). 
    \label{eqn:optimal-transport-order-constraint2}
\end{equation}

Problem \eqref{eqn:optimal-transport-order-constraint2} is  convex   with linear costs and constraints. However,  $\nijdx{\order}{i}{j}{k}$ adds a significant number 
of constraints  to the $m+n$ in the standard OT formulation  \eqref{eqn:optimal-transport}.
The $\inf$ in~\eqref{eqn:transport-plan-ordering-constraints-3} and~\eqref{eqn:optimal-transport-order-constraint2} are replaced by $\min$ when feasible; we address such sufficient conditions in the following.

\paragraph{ Feasibility of the  Proposed OT with OC Problem}
 $\tpoly$  is  feasible since $ab^T \in \tpoly$,~\citep[see][]{cuturi_sinkhorn_2013}.
In certain instances it may be considered ambiguous to match a single position to multiple  positions, consider variates $\nij{i}{j}{k}$ where $\ni{i}{k}=i_1,i_2,\cdots, i_k$ (and $\ni{j}{k}$) do not repeat row (or column) indices; let $\indicate{\ni{i}{k}}{p}=1$  if $p=i_\ell$ for at most one choice of $\ell$, or $0$ otherwise, and similarly 
$\indicate{\ni{j}{k}}{q}$.
We next derive conditions for the feasibility of~\eqref{eqn:optimal-transport-order-constraint2} under mild assumptions.

\begin{proposition}
    \label{prop:optimal-transport-feasibility}
    
    Suppose $\ni{i}{k}=i_1,i_2,\cdots, i_k$ (and $\ni{j}{k}$) do not repeat row (or column) indices. Further suppose $\tplan \in \mnReal$, 
    where $\ijdx{\tplan}{i_\ell}{j_\ell}= c_\ell$
    for $\ell \in \iset{k}$, and $\ijdx{\tplan}{p}{q}$ for $\ij{p}{q} \in \langidx$, resp. satisfy:
    \begin{eqnarray}
        0~\leq~\ijdx{\tplan}{i_\ell}{j_\ell} &=& c_\ell ~~\leq ~~\min(a_{i_\ell}, b_{j_\ell}), \label{eqn:feas1a}
      \\
        \!\!\!\!\!\! 0~\leq~\ijdx{\tplan}{p}{q} \!\!\! &=&  \!\!\!
        \frac{
            (a_p -\indicate{\ni{i}{k}}{p} c_p ) \cdot (b_q -\indicate{\ni{j}{k}}{q} c_q)
        }{\alpha(c_1,\cdots,c_\ell)}
          \label{eqn:feas1b}
    \end{eqnarray}
    where
    $\alpha = \alpha(c_1,\cdots,c_\ell) = 1 - \sum_{\ell \in \iset{k}} c_\ell \geq 0$.
     Then, $\tplan$ is  feasible  w.r.t~\eqref{eqn:optimal-transport-order-constraint2} if:
    \begin{equation}
         \frac{a_p b_q }{\alpha(c_1, \cdots, c_k)} \leq c_k \leq \cdots \leq c_1~~\mbox{for all}~~\ij{p}{q} \in \langidx.
         \label{eqn:feas}
    \end{equation}
\end{proposition}
\begin{proof}
   Since $\ni{i}{k}$ and $\ni{j}{k}$ do not repeat row/column indices, 
   then~\eqref{eqn:feas1a} and~\eqref{eqn:feas1b} imply that
  $\tplan$ satisfies $\tplan \mOnes{n} = a$ and $\tplan^T \mOnes{m} = b$; and since they also imply $\tplan \geq 0 $, thus $\tplan \in \tpoly$.
   Furthermore, we also have
   $\tplan \in\nijdx{\order}{i}{j}{k} $, because~\eqref{eqn:feas1b} and~\eqref{eqn:feas} imply for all $\ij{p}{q} \in \langidx$,
   that $\ijdx{\tplan}{p}{q} \leq a_p b_q / \alpha \leq c_k $. 
\end{proof}

\begin{corollary}
    \label{cor:optimal-transport-feasibility}
    
    Suppose $\ni{i}{k}$ (and $\ni{j}{k}$) do not repeat row (or column) indices. 
    If $a_i=1/m$ and $b_j=1/n$, then $\tplan$ is feasible w.r.t~\eqref{eqn:optimal-transport-order-constraint2} if $
        \min(m,n) \geq (1 - k / \max(m,n))^{-1} 
    $.
\end{corollary}
\begin{proof}
    For all $\ell \in \iset{k}$, pick $c_\ell = \min(1/m, 1/n) = c$ for some $c \geq 0$; this choice for $c$ satisfies~\eqref{eqn:feas1a} and~\eqref{eqn:feas1b}. 
    Then~\eqref{eqn:feas} holds if 
    $
        a_p b_q / c \leq 1 - k c = \alpha
    $
    holds, or equivalently
    $
        \min(m,n) \geq (1 - k / \max(m,n))^{-1} 
    $ holds.
\end{proof}

Cor.~\ref{cor:optimal-transport-feasibility} states for uniform constraints $a_i=1/m$ and $b_j=1/n$, feasibility typically holds when $k \ll \max(m,n)$, since an upper bound for
$
   (1 - k / \max(m,n))^{-1} 
$
is
$
   1 + \bigO{k / \max(m,n)}
$.
Indeed, such is the case of interest where a sparse $k$ number of OCs are enforced, see~\eqref{eqn:transport-plan-ordering-constraints}.



\section{Solving OT  with Order Constraints}
\label{sec:theory}

In this section we provide a $\delta$-approximate method for \eqref{eqn:optimal-transport-order-constraint2} that runs in
$\bigO{\norm{D}{\infty}  /\delta \cdot  mn \log mn}$ time using the alternating direction method of multipliers (ADMM) \citep{admm,beck_first-order_2017}.
In what follows assume feasiblity of~\eqref{eqn:optimal-transport-order-constraint2}, and 
the strategy is to consider the (non-empty) polytope $\tpoly \intersect  \nijdx{\order}{i}{j}{k}$ as the intersection of two closed convex sets:
\begin{eqnarray}
    \constraintThree &=& \{
        X \in \mnReal: 
        X \mOnes{n} = a,  
        X^T \mOnes{n} = b
    \},
    \nonumber
\end{eqnarray}
and $\constraintOne= \nijdx{\order}{i}{j}{k} $,
used to equivalently rewrite~\eqref{eqn:optimal-transport-order-constraint2} as:
\begin{eqnarray}
    &\min_{
        X \in \constraintThree
    } 
    f(X)
    +
    \indicate{\constraintOne}{
        Z
    },
    \label{eqn:admm-reformulation1} \\
    &\mbox{s.t.}~~~X,Z\in \mnReal,~~X - Z = \mathbf{0},
    \label{eqn:admm-reformulation2}
\end{eqnarray}
where the indicator function $\indicate{\mathcal{C}}{X} = 0$ if $X \in \mathcal{C}$, and $\indicate{\mathcal{C}}{X} = \infty$ otherwise.
For $X \in \constraintThree$,  $Z \in \constraintOne$, the equality implies $X=Z=\tplan$ for feasible $\tplan \in \tpoly$.
The first-order, iterative ADMM procedure is summarized in Algorithm~\ref{alg:admm};
$\rho > 0$ is a penalty term, and iterates $X_{t+1}, Z_{t+1}$ solve, respectively:
\begin{eqnarray}
    & \min_{
        X \in \constraintThree
    }
    \left(
        \trace{D^T X} 
        + \frac{\rho}{2}
        \norm{
            X - Z_t + M_t
        }{F}^2
    \right) \nonumber \\
    & \min_{
        Z \in \mnReal
    }
    \left(
        \indicate{\constraintOne}{Z}
        + \frac{\rho}{2}
        \norm{
            X_{t+1} - Z + M_t
        }{F}^2
    \right) 
    \label{eqn:admm}
\end{eqnarray}
where  $M_t$ is the (scaled) dual iterate, see~\cite{admm}. 
Simplifying the first $\min$ expression in \eqref{eqn:admm} shows that $X_{t+1}$ is nothing but the Euclidean projection onto $\constraintThree$ of a point $Z_t - M_t - \rho^{-1} D$; similarly $Z_{t+1}$ is the Euclidean projection onto $\constraintOne $ of $X_{t+1} + M_t$.

\begin{algorithm}[H]
    \caption{
        Iterative procedure for OT under OC $\nijdx{\order}{i}{j}{k}$ with linear costs $f(X) = \trace{D^T X}$.
    }
    \label{alg:admm}
    \begin{algorithmic}[1]
        \REQUIRE Costs $D$, penalty  $\rho$, initial  $\tidx{M}{0},\tidx{Z}{0}$.
        \FOR{round $t \geq 1$ until  stopping}
            \STATE Update $X_{t+1} = \proj{\constraintThree}{Z_t - M_t - \rho^{-1} D}$ 
            \STATE Update $Z_{t+1} = \proj{\constraintOne}{X_{t+1} + M_t}$. 
            \STATE Update (scaled) dual variable $M_{t+1} = M_t + X_t - Z_t$.
        \ENDFOR
        \STATE Return  $X_t$
    \end{algorithmic}
\end{algorithm}

We next derive the projections, $\proj{\constraintThree}{\cdot}$ and $\proj{\constraintOne}{\cdot}$, required for Alg.~\ref{alg:admm}.
Prop. \ref{prop:euclidean-projector-sums} gives the Euclidean projector for $\constraintThree$.
 
\begin{proposition}[Projection $\constraintThree$]
    \label{prop:euclidean-projector-sums}

    For $\constraintThree$ consider the Euclidean projector $\proj{\constraintThree}{X}$. 
    Let $\meanproj{k}$ be the projection of a vector onto its mean,~\ie, $\meanproj{k} = \frac{1}{k} \mOnes{k} \mOnes{k}^T$.
    Define  matrices $M := \mIdentity{m} - \frac{m}{m+n} \meanproj{m}$ and $N := \mIdentity{n} - \frac{n}{m+n} \meanproj{n}$, and  matrices
    \begin{eqnarray}
        Y_1 &:=& \frac{1}{n} 
            \overbrace{
                \left[ 
                    Ma, \cdots, Ma
                \right]^T
            }^{m~\mbox{copies (rows)}} 
            + 
            \frac{1}{m}
            \overbrace{
                \left[ 
                    Nb, \cdots, Nb
                \right]
            }^{n~\mbox{copies (cols)}},~~~~~
        \nonumber
        \\
        Y_2 &:=& M (XP_n) + (P_m X) M.
        \label{eqn:rowcolsum-Y2}
    \end{eqnarray}
    Then the projection $\hat{X} = \proj{\constraintThree}{X}$ satisfies
   $
        \hat{X} = \tidx{Y}{1} + (X - \tidx{Y}{2}).
    $
\end{proposition}
The proof of Prop.~\ref{prop:euclidean-projector-sums}
can be found in the Appendix.
\citet{wang_learning_2010} proved a similar property for the simplified setting where  $a = b = \mOnes{\cdot}$ and $X$ symmetric, which does not hold in the optimal transport setting.  

Next we address computing $\proj{\constraintOne}{X}$ for $\constraintOne =\nijdx{\order}{i}{j}{k}$ for values of $k$ of interest.
For the special case where $k= mn$,~\citet{grotzinger_projections_1984}  defined the so-called Pool Adjacent Violators Algorithm (PAVA) that solves $\proj{\constraintOne}{X}$ for
$\constraintOne = \nijdx{\order}{i}{j}{mn}$.
We cannot directly make use of PAVA for OT with Order Constraints because we require only a small number, $k \ll mn$, of order constraints, and not the full set, since the OT itself should learn the ordering of the other coefficients.

However, the idea behind PAVA is of interest. To this end we  extend PAVA, and call this extension ePAVA. ePAVA  projects into $\constraintOne =\nijdx{\order}{i}{j}{k}$ for general $k$, and is presented in
Alg.~\ref{alg:pava}. ePAVA makes use of two quantities defined below; the proof of correctness is given in the Appendix.

We define a block $B$ to be a partition of indices. Let $\texttt{le}[B]$ denote the left  and  $\texttt{ri}[B]$ the right boundary of $B$; $\texttt{val}[B]$ is the value of its projection.
For any $\eta \geq 0$, let 
$r(\eta)$ denote the 
$ \rank{\ijdx{X}{i_1}{j_1} - \eta }$
with respect to decreasing order $\odidx{X}{1} \geq \odidx{X}{2} \geq \cdots \geq \odidx{X}{mn-k+1}$ over all variates in $\langidx \union \{\ij{i_1}{j_1}\}$, for $\langidx$ in~\eqref{eqn:transport-plan-ordering-constraints}.
We require averages $\ijdx{\pmean}{p}{q} := \sum_{\ell = p}^q \ijdx{X}{i_\ell}{j_\ell} / (q-p+1)$ and a threshold function
$T(\cdot) := \zeroprog{\tau(t(\eta), \eta)}$, where $\tau(s,\eta)$ for any $0 \leq s < r(\eta)$ and \emph{integer} $t(\eta) \geq 0$ are defined\footnote{
    The sum in \eqref{eqn:order-threshold-tau} equals $0$ when $s=0$.
} as:
\begin{eqnarray}
    \tau(s,\eta) 
    \!\! &=& \!\!
    \frac{1}{s+1} \left(
        \ijdx{X}{i_1}{j_1} - \eta + \sum_{\ell=1}^s \odidx{X}{\ell}
    \right),
    \label{eqn:order-threshold-tau}
    \\
    t(\eta) + 1 
    \!\! &=& \!\!
    \arg\min \left\{ 
        s \in \iset{r} : \tau(s,\eta) > X_{(s)}
    \right\}
    \label{eqn:order-threshold-t}
\end{eqnarray}
or $t(\eta) + 1 = r(\eta)$ whenever the set in \eqref{eqn:order-threshold-t} is empty.

\begin{algorithm}[!t]
    \caption{ePAVA for $\constraintOne=\nijdx{\order}{i}{j}{k}$ for $k \in \ijset{m}{n}$}
    \label{alg:pava}
    \begin{algorithmic}[1]
        \REQUIRE $X \in \mnReal$. Indices $\nij{i}{j}{k}$.
        \STATE  $\ell :=1, \bl := 1$, $\texttt{le}[1] := \texttt{ri}[1] := 1$, $\texttt{val}[1] := T(0)$. 
        \FOR{$\ell \leq k$}
            \STATE  $\bl:=\bl+1$, $\ell:=\ell+1$, $\texttt{le}[\bl] := \texttt{ri}[\bl] := \ell$,  $\texttt{val}[\bl] := \ijdx{X}{i_\ell}{j_\ell}$.
            \FOR{$\bl \geq 2$ and $\texttt{val}[\bl] \leq \texttt{val}[\bl-1]$}
                \STATE Let $q = \texttt{ri}[\bl]$.
                \IF{$\bl=2$}
                    \STATE Solve and store $\tilde{\eta} \geq 0$ satisfying $T(\tilde{\eta}) = \ijdx{\pmean}{2}{q} + \tilde{\eta} / (q - 1)$. 
                    Set $\texttt{val}[\bl-1] := T(\tilde{\eta})$.
                \ELSE
                    \STATE Set $\texttt{val}[\bl-1] := \ijdx{\pmean}{p}{q}$ for $p = \texttt{ri}[\bl-1]$.
                \ENDIF
                \STATE Set $\texttt{ri}[\bl-1] := \texttt{ri}[\bl]$. Decrement $\bl := \bl - 1$.
            \ENDFOR
        \ENDFOR
        \STATE Return $\bl, \tilde{\eta}$, \texttt{le}, \texttt{ri}, and \texttt{val}.
    \end{algorithmic}
\end{algorithm}

\begin{proposition}[Projection $\constraintOne$]
    \label{prop:euclidean-projector-order}
    For $\constraintOne = \nijdx{\order}{i}{j}{k}$ for any 
    $\ij{i_\ell}{j_\ell} \in \ijset{m}{n}$ where $ \ell \in \iset{k}$,
    consider the Euclidean projector $\proj{\constraintOne}{X}$ for any $X \in \mnReal$.
    Let $T(\eta) := \zeroprog{\tau(t(\eta),\eta)}$ for 
    $\tau,t$ in~\eqref{eqn:order-threshold-tau} and~\eqref{eqn:order-threshold-t} and $\langidx$ in~\eqref{eqn:transport-plan-ordering-constraints}.
    Then for any $X \in \mnReal$, ePAVA will successfully terminate with some $\bl, \tilde{\eta}$, \textup{\texttt{le}}, \textup{\texttt{ri}}, and \textup{\texttt{val}}. 
    Furthermore, the projection $\hat{X} = \proj{\constraintOne}{X}$ 
    satisfies i) for $\ij{p}{q} \in \langidx$ we have $\ijdx{\hat{X}}{p}{q} = T(\tilde{\eta}) = \textup{\texttt{val}}[1]$ if $\rank{\ijdx{X}{p}{q}} \leq t(\tilde{\eta})$ or $\ijdx{\hat{X}}{p}{q} = \zeroprog{\ijdx{X}{p}{q}}$ otherwise, and 
    ii) for  $\ell \in \iset{k}$
    we have $\ijdx{\hat{X}}{i_\ell}{j_\ell} = \textup{\texttt{val}}[\bl']$ iff $\textup{\texttt{le}}[\bl'] \leq \ell \leq \textup{\texttt{ri}}[\bl']$ for some $\bl' \in \iset{\bl}$.
\end{proposition}

\paragraph{Convergence Analysis}
Standard ADMM convergence results hold as follows.
As $t \rightarrow \infty$, the primal residue $\norm{X_t - Z_t}{F} \rightarrow 0$ and dual residue $\rho \norm{Z_{t-1} - Z_{t}}{F} \rightarrow 0$, which can be used to show asymptotic convergence of $f(X_t)$ to the optimum~\citep[see][p. 17]{boyd_distributed_2010}.
The iteration complexity of Alg.~\ref{alg:admm} can be obtained from the following\footnote{
  Use Thm 15.4 of~\citet{beck_first-order_2017} with $\mathbf{G},\mathbf{Q}$ set to zero, and checking Assump 15.2 holds for $f(X)$ and $\constraintThree,\constraintOne$. 
}
non-asymptotic result~\citep[see][Thm 15.4]{beck_first-order_2017}.
Consider the following.

As~\eqref{eqn:admm-reformulation1}-\eqref{eqn:admm-reformulation2} are both linear, if $\constraintAll \neq \emptyset$, then an optimal $Y^* = \arg \max_{Y \in \mnReal} d(Y)$  exists~\citep{boyd_convex_2004} from solving the dual  $d(Y) :=$
\begin{equation}
    \min_{X \in \constraintThree, Z \in \constraintOne}
    \trace{X^T(D+Y)} - \trace{Y^T Z} 
    \label{eqn:dual-function}
\end{equation}
Recall Proposition~\ref{prop:optimal-transport-feasibility} is a sufficient condition for $\constraintAll \neq \emptyset$.

\begin{theorem}[Convergence  of Algorithm~\ref{alg:admm},~\citet{beck_first-order_2017}]
\label{thm:convergence-admm}
Let $X_t,Z_t,M_t$ be the sequences of iterates from Algorithm~\ref{alg:admm} with penalty $\rho > 0$, and let
$\bar{X}_t = (t+1)^{-1} \sum_{\ell=1}^{t+1} X_{\ell}$ and similarly $\bar{Z}_t$.
Assume $\constraintAll \neq \emptyset$, and let $f^*$ and $\tplan^*$ denote the optimum value and solution respectively of \eqref{eqn:optimal-transport-order-constraint2}, and $Y^* = \arg \max_{Y \in \mnReal} d(Y)$ for $d(Y)$ in \eqref{eqn:dual-function}.
Then we have the convergence bounds
\[
    f(\bar{X}_t) - f^* \leq \frac{\rho c_1} {2(t+1)},~~~
    \norm{\bar{X}_t - \bar{Z}_t}{F} \leq \frac{\rho c_1} {c_2(t+1)}
\]
where $c_1 \geq \norm{Z_0 - \tplan^*}{F}^2 + (\norm{M_0}{F} + c_2/\rho)^2 $ and $c_2 \geq 2\norm{Y^*}{F}$.
\end{theorem}

\begin{proposition}[Iteration complexity]
    \label{prop:iteration-complexity}
Algorithm~\ref{alg:admm}, given linear costs $D$ and   penalty parameter $\rho \geq 0$, initialized with $Z_0 = M_0 = \mathbf{0}$, achieves error $f(\bar{X}_t) - f^* \leq \delta$, in $\bigO{\norm{D}{\infty}/\delta}$ iterations.
\end{proposition}
\begin{proof}
We seek to estimate $\rho c_1$ that bounds the error $f(\bar{X}_t) - f^*$ in Thm. \ref{thm:convergence-admm}.
   First, consider the constant $c_1$.
   Setting $Z_0 = M_0 = \mathbf{0}$ gives $\norm{Z_0 - \tplan^*}{F}^2 \leq 1$ and  $\norm{M_0}{F}=0$, then Thm. \ref{thm:convergence-admm} requires
$c_1 \geq 1 + (c_2/\rho)^2$;
 hence if we set the penalty $\rho = c_2$ then $c_1 \geq 2$. 
 Next, it remains to estimate $\rho$, or equivalently, $c_2$.
 Thm. \ref{thm:convergence-admm} requires $c_2$ to be at least the largest possible norm that an optimizer of the dual \eqref{eqn:dual-function} can take. 
 To this end if, writing $Y^*(D')$ for the optimizer of \eqref{eqn:dual-function} given costs $D'$,
 it suffices to consider a constant $c$ satisfying $c \geq \norm{Y^*(D')}{F}$ for any normalized
 \footnote{
    The dual $d$ in \eqref{eqn:dual-function} is scale-invariant, in the sense that if
    $Y^*$ optimizes \eqref{eqn:dual-function} for costs $D$, then for any $\alpha > 0$,
    we have $\alpha Y^*$ optimizes the same for 
    costs $\alpha D$.
} 
cost $D'$  with $\norm{D'}{\infty} = 1$, and then picking $c_2$ such that $c_2 \geq 2 \norm{D}{\infty} c $.
Hence we arrive at an estimate for $\rho c_1$ to be 
in $\bigO{\norm{D}{\infty}}$. 
Thus we conclude the error is at most $\delta$ under the claimed number of iterations.
\end{proof}


Prop.~\ref{prop:iteration-complexity} holds for the choice\footnote{
 For the choice $\rho=c_2$, it follows that the error $\norm{\bar{X}_t - \bar{Z}_t}{F} $ in Thm.~\ref{thm:convergence-admm} also drops linearly with $t$.
} $\rho=c_2$, but however $c_2$ is unknown in practice; the common adage is to simply set $\rho=1$,~\citep[see][]{admm}.
Invoking both Propositions~\ref{prop:iteration-complexity} and~\ref{prop:complexity} below, we estimate the total operation complexity as $\bigO{\norm{D}{\infty} /\delta \cdot  mn \log mn}$.

\begin{proposition}
    \label{prop:complexity}
    
    One round of Algorithm~\ref{alg:admm} runs in $\bigO{ mn \log(mn)}$ time.
\end{proposition}
\begin{proof}
   Line 2 costs $\bigO{mn}$ for $\constraintThree$ and Line 4 costs $\bigO{k + (mn-k) \log (mn-k)}$ for $\constraintOne$. 
   The update in Line 2 requires only $\bigO{mn}$.
   Note that for $ \proj{\constraintThree}{\cdot}$ the expressions~\eqref{eqn:rowcolsum-Y2} require only matrix-vector multiplications with $\bigO{mn}$ complexity, as follows.
   For $Y_1$ this is clear from~\eqref{eqn:rowcolsum-Y2}. 
For $Y_2$, consider the left term $M(XP_n)$; then the other term $(P_m X) M$ similarly follows.
Putting $M = \mIdentity{m} - \frac{m}{m+n} \meanproj{m}$ and $P_k = \frac{1}{k}\mOnes{k}\mOnes{k}^T$ for $k=m,n$, we get that $M(XP_n)$ can be expanded (ignoring scaling factors) into two terms $X \mOnes{n} \mOnes{m}^T$ and $\mOnes{m}^T (\mOnes{m} X \mOnes{n}) \mOnes{n}^T$. Hence  the claim holds given that $\mOnes{m}, \mOnes{n}$ are vectors.

  For $\constraintOne$, observe that computing $T(\eta) = \tau(t(\eta),\eta)$ multiple times in Alg.~\ref{alg:pava} (Line 7),  requires a one-time sort of $mn-k+1$ terms and  Lines 5-10 incur an additional  $\bigO{k}$  following arguments in~\citet{grotzinger_projections_1984}.
\end{proof}

Care should to be taken in comparing our  complexity, $\bigO{\norm{D}{\infty} /\delta \cdot  mn \log mn}$, to that of   OT methods that use  rounding approximations  \citet{altschuler_near-linear_2017}
to ensure feasibility of each iterate. Such methods, which use KL-divergence, are $\bigO{n^2 / \delta \cdot \norm{D}{\infty} }$ for $m=n$, 
\citep[see][]{guminov_combination_2021,jambulapati_direct_2019}.
For order constraints like other structured OT approaches~\citep[e.g.][]{scetbon_low-rank_2021},
a rounding method for $\tpoly \intersect \nijdx{\order}{i}{j}{k}$ is not available. However, we have the bound on $\norm{\bar{X}_t - \bar{Z}_t}{F}$ in Thm.~\ref{thm:convergence-admm} to quantify the feasibility of the iterates $X_t,Z_t$ from $\constraintAll$.



\section {Explainability via Branch-and-Bound}

We have thus far assumed that the structure  captured by the  optimal transport plan via  order constraints has been provided externally. In some settings, the structure is not provided and one must estimate the most important variates to define the corresponding constraints. With the goal of generating  plans from which a one can select, we propose an explainable and efficient approach using branch \& bound  to compute \emph{a diverse set of  optimal transport plans}. 

The approach successively computes a bound on the best score that can be obtained with the variates currently fixed;  if it cannot improve the best known score, the branch is cut  and  the next branch, i.e. the next set of fixed variates, is explored. 
Consider the $k$ order constraints~\eqref{eqn:transport-plan-ordering-constraints} enforced by the variate $\nij{i}{j}{k}$, and further consider  introducing an additional variate $\ij{i}{j} \in \ijset{m}{n} $ such that 
$
\ijdx{\tplan}{i_k}{j_k} \geq \cdots \geq \ijdx{\tplan}{i_1}{j_1} 
\geq \ijdx{\tplan}{i}{j}
\geq 
\ijdx{\tplan}{p}{q}
$
for 
$
\ij{p}{q} \in \langidx
$,
see~\eqref{eqn:transport-plan-ordering-constraints}.
This manner of introducing variates one-at-a-time, can be likened to variable selection  and formalized by a tree structure. On the tree $\searchtree$,  each node $\nij{i}{j}{k}$ has 
children $\nij{i'}{j'}{k+1}=\ij{i}{j},\mij{i}{j}{k}$ that share the same top-$k$ constraints;  
the root node corresponds to unconstrained OT, and has children at level $k=1$ corresponding to single order constraint variates $\nij{i}{j}{1}=\ij{i}{j}$.
We would like the procedure to always select all ancestors of a given node before selecting the node itself.
Hence, the root node should be selected first.

The branch selection aimed at increasing diversity of the plans is guided by two threshold parameters $\tau_1, \tau_2$ that lie between 0 and 1.
For any given node $\nij{i}{j}{k}$, the parameters $\tau_1, \tau_2$ limit its children to those that only deliver reasonably likely transport plans,
as follows.


\begin{algorithm}[!t]
    \caption{
        Learning subtree $\esttree(k_1,k_2,k_3,\tau_1, \tau_2)$ of $\searchtree(k_3,\tau_1, \tau_2)$ 
        and top-$k_2$ candidate plans
        for linear costs $f(\tplan) =  \trace{D^T \tplan}$.
    }
    \label{alg:heuristic-search}
    \begin{algorithmic}[1]
        \REQUIRE Costs $D$, 
        thresholds $0 \leq \tau_1, \tau_2 \leq 1$. 
        Search upper limit $k_1$, number of top candidates $k_2$, and search depth $k_3 \leq \min(m,n)$.
        \STATE Compute $\tidx{\hat{\tplan}}{1}$ using~\eqref{eqn:optimal-transport}.
        Init $\esttree(k_1,k_2,k_3,\tau_1,\tau_2)$.
        \STATE Use $\tidx{\hat{\tplan}}{1}$, $\tau_1,\tau_2$ in 
        ~\eqref{eqn:heuristic-statistic-triple}  and~\eqref{eqn:heuristic-statistic}
        to obtain $\variates$ and $\ijdx{\heuristic}{i}{j}$. 
        Init.  $\stack = \{ (\ij{i}{j}, \ijdx{\heuristic}{i}{j}): \ij{i}{j} \in \variates \}$.
         \texttt{count=0}.
        \FOR{\texttt{count} $< k_1$}
            \STATE Pop $\nij{i}{j}{k}$ having smallest $\heuristic$ in $\stack$, for some $k$ constraints. Compute
             $\mathcal{L}$ from  right-hand side of~\eqref{eqn:final-bound}. 
            \IF{$\tidx{\hat{\tplan}}{k_2}$ is not yet obtained or $\mathcal{L} > f(\tidx{\hat{\tplan}}{k_2}) $ }
                \STATE Solve Algorithm~\ref{alg:admm} with order constraint $\nijdx{\order}{i}{j}{k}$ for new candidate $\hat{\tplan}$.
                Set \texttt{count += 1}.
                \STATE 
                Update top-$k_2$ candidates $\tidx{\hat{\tplan}}{1}, \tidx{\hat{\tplan}}{2}, \cdots, \tidx{\hat{\tplan}}{k_2}$ and $\esttree(k_1,k_2,k_3,\tau_1,\tau_2)$ using new candidate $\hat{\tplan}$.
            \ENDIF
            \IF{$k$ equals $k_3$}
                \STATE Go to line 4.
            \ENDIF
            \IF{$\tidx{\hat{\tplan}}{k_2}$ not yet obtained or $f(\hat{\tplan}) < f(\tidx{\hat{\tplan}}{k_2}) $ }
                \STATE Use $\hat{\tplan}$, $\tau_1,\tau_2$ in 
        ~\eqref{eqn:heuristic-statistic-triple}  and~\eqref{eqn:heuristic-statistic} and obtain
                new variates $\ij{i}{j} \in \variates(\hat{\tplan})$ and $\{\ijdx{\heuristic}{i}{j}\}_{\ij{i}{j} \in \variates(\hat{\tplan})}$.
                \FOR{variate $\ij{i}{j}$ in $\variates(\hat{\tplan}) $}
                    \IF{
                        $i \notin \ni{i}{k}$ and
                        $j \notin \ni{j}{k}$
                    }
                        \STATE Push $ (\ij{i}{j},\ij{i_1}{j_1},\cdots,\ij{i_k}{j_k} , \ijdx{\heuristic}{i}{j})$ onto stack $\stack$.
                    \ENDIF
                \ENDFOR
            \ENDIF
        \ENDFOR
        \STATE Return top $k_2$ candidates $\tidx{\hat{\tplan}}{1}, \tidx{\hat{\tplan}}{2}, \cdots, \tidx{\hat{\tplan}}{k_2}$ and $\esttree(k_1,k_2,k_3,\tau_1,\tau_2)$.
    \end{algorithmic}
\end{algorithm}

For example, suppose that $\tplan \in \tpoly$ is the solution of~\eqref{eqn:optimal-transport-order-constraint2} 
obtained from the variate $\nij{i}{j}{k}$.
To determine the children of $\nij{i}{j}{k}$, first
compute saturation levels normalized between $0$ and $1$ given as $\ijdx{\saturation}{i}{j} =  \ijdx{\tplan}{i}{j} / \min(a_i,b_j)$, to derive
\begin{eqnarray}
    \left(
        \ijdx{\selfheu}{i}{j},
        \ijdx{\rowheu}{i}{j},
        \ijdx{\colheu}{i}{j}
    \right)
    &=&
    \left(
        \ijdx{\saturation}{i}{j},
        \max_{\ell \in \iset{n}:~\ell\neq i} \ijdx{\saturation}{i}{\ell},
        \max_{\ell \in \iset{m}:~\ell\neq j} \ijdx{\saturation}{\ell}{j}
    \right)
    \nonumber \\
    \ijdx{\heuristic}{i}{j} &:=& \min(
        \ijdx{\rowheu}{i}{j},
        \ijdx{\colheu}{i}{j}
    ).
    \label{eqn:heuristic-statistic-triple}
\end{eqnarray}
Low saturation values in~\eqref{eqn:heuristic-statistic-triple} imply uncertainty in the assignments $\tplan$. 
The thresholds $\tau_1, \tau_2$ are used to determine the set $\variates(\tplan)$ that filters those variates   $\ij{i}{j}$
with low self saturation ($\tau_1$) and low neighborhood  saturation ($\tau_2$):
\begin{equation}
    \variates(\tplan) := \{
        \ij{i}{j} \in \ijset{m}{n}:  \ijdx{\selfheu}{i}{j} \leq \tau_1, \ijdx{\heuristic}{i}{j} \leq \tau_2
    \}. \label{eqn:heuristic-statistic}
\end{equation}
The children of $\nij{i}{j}{k}$ are obtained as $\nij{i'}{j'}{k+1} = \ij{i}{j},\nij{i}{j}{k} $ for all $\ij{i}{j} \in \variates(\tplan)$.

The proposed explainable approach adds diversity  by learning the top plans. Let us define $k_1, k_2, k_3$ as follows: The  number of nodes  constructed while learning a subtree is at most $k_1$, the number of top plans to retain is $k_2$, and $k_3$, as noted above, limits the tree depth. 
The branch-and-bound search proceeds on the reduced tree given by
 $\searchtree(k_3, \tau_1, \tau_2)$, of depth $k_3$, 
 containing nodes~\eqref{eqn:heuristic-statistic} whose saturation lie beneath the thresholds $\tau_1,\tau_2$.
 We thus wish to learn the
top-$k_2$  plans given by  subtree, $\esttree(k_1,k_2,k_3,\tau_1, \tau_2)$, and
 identifying the nodes of $\searchtree(k_3, \tau_1, \tau_2)$ that are in 1-to-1 correspondence with those top-$k_2$  plans.
The tree $\searchtree(k_3, \tau_1, \tau_2)$ is constructed dynamically, and 
 avoids redundantly computing plans $\tplan$ corresponding to variates $\nij{i}{j}{k}$ if the $k_2$-best candidate  has cost lower than either: i) a known lower bound $\mathcal{L} \leq f(\tplan) $, or ii) the cost of $\nij{i}{j}{k}$ parent's plan.
Finally, in the construction of $\searchtree(k_3, \tau_1, \tau_2)$ we conly consider variates
where $\ni{i}{k}$ and $\ni{j}{k}$ do not repeat indices.

Alg.~\ref{alg:heuristic-search} summarizes the branch-and-bound variable selection procedure.
Upon termination, a diverse set of at most $k_1$ plans are computed   and the $k_2$ plans identified by $\esttree(k_1,k_2,k_3,\tau_1, \tau_2)$ are the top-$k_2$ amongst these $k_1$ plans, as the most uncertain variates are handled first.


\begin{figure}
    \centering
    \includegraphics[width=.95\columnwidth]{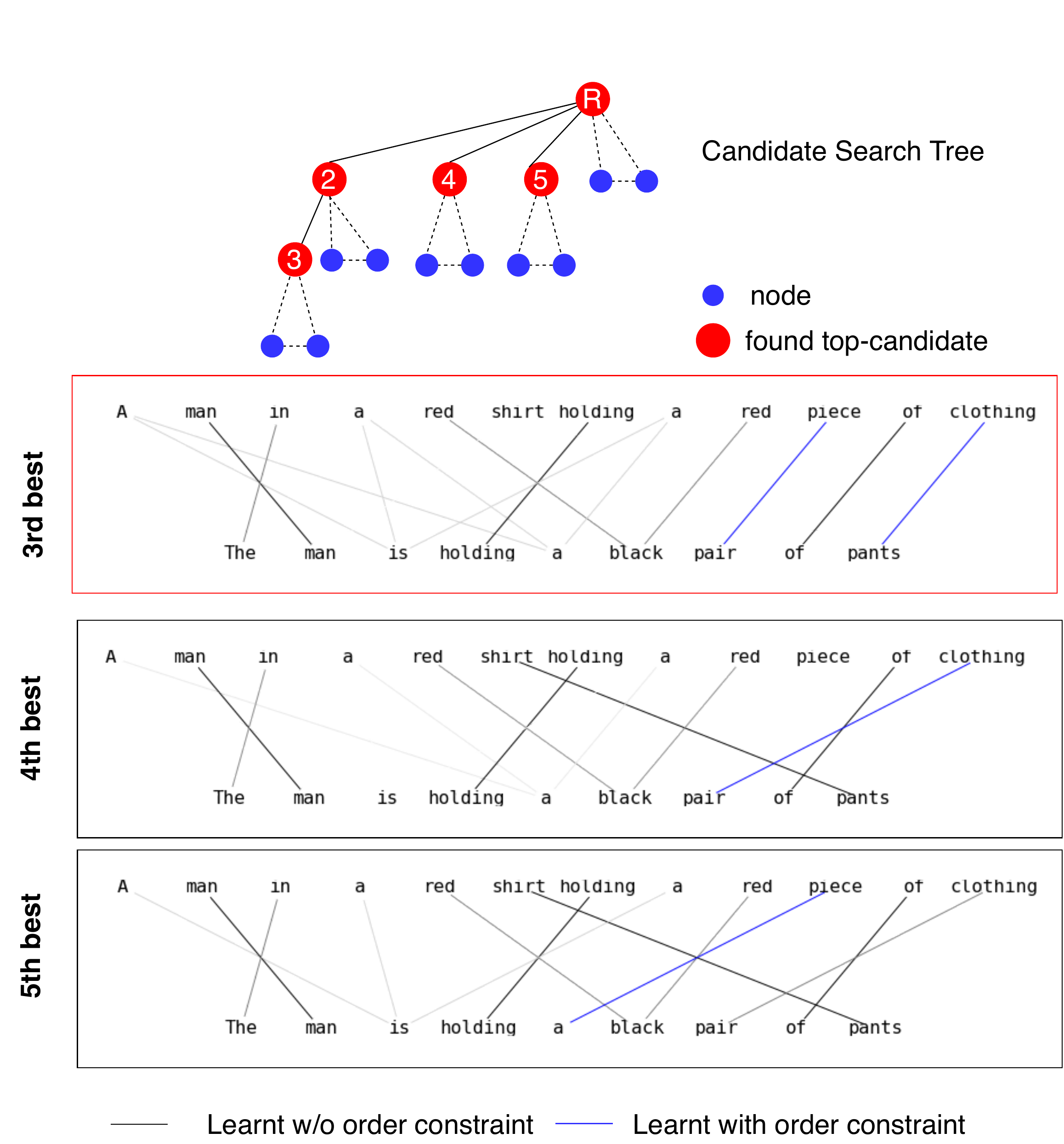}
    \caption{Top $k_2=5$ candidates for OT with multiple order constraints for the example in Fig.~\ref{fig:intro-example}}
    \label{fig:multi-oc-search-tree}
\end{figure}

\begin{figure*}
    \centering
    \includegraphics[width=0.85\textwidth]{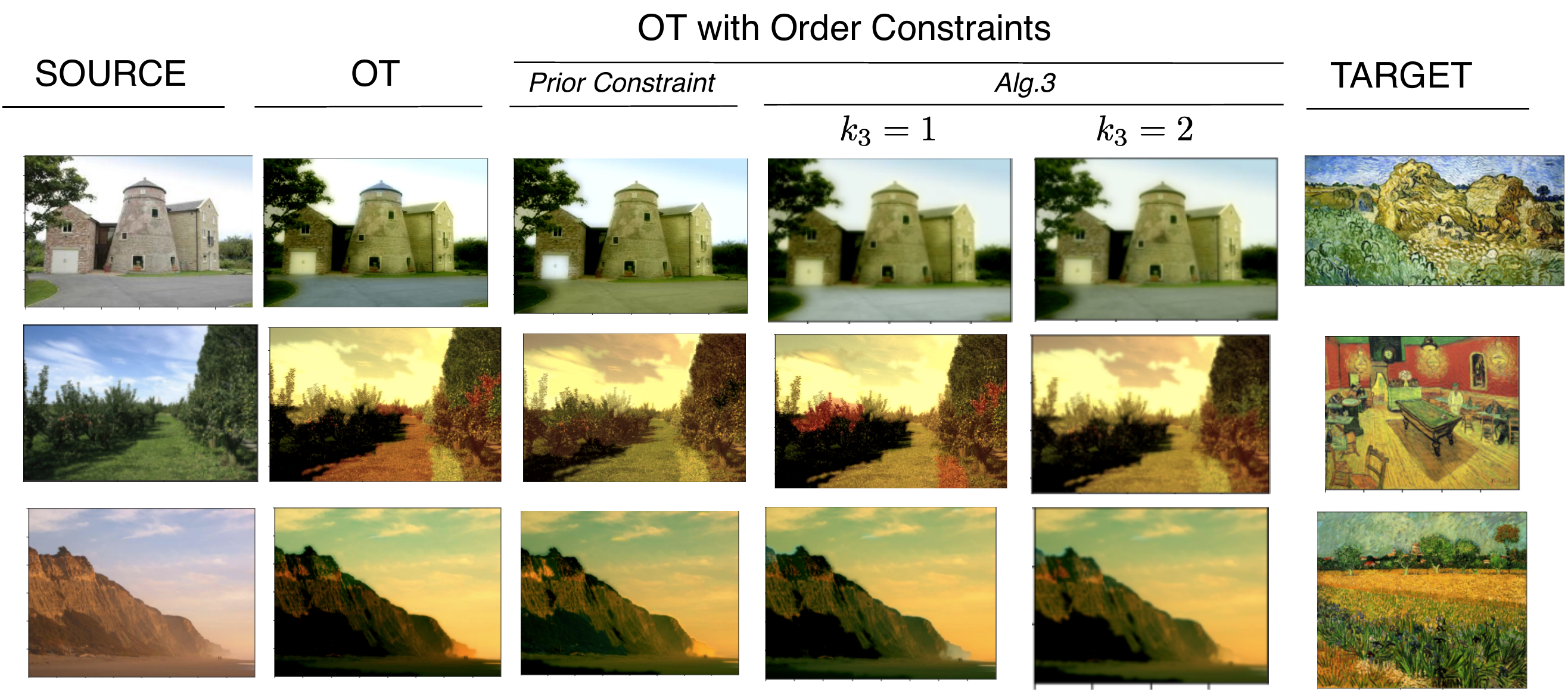}
    \caption{Color transfer candidates obtained using OT with OC (Alg.~\ref{alg:heuristic-search})}
    \label{fig:expr-color-transfer}
\end{figure*}

It remains  to define the lower bound $\mathcal{L} \leq f(\tplan) $ (Line 5); \citet{kusner_word_2015} proposed a bound for~\eqref{eqn:optimal-transport}, which we extend
to the order constrained case~\eqref{eqn:optimal-transport-order-constraint2} as follows:
\begin{eqnarray}
    &\min_{
        \tplan \in \tpoly \intersect 
        \nijdx{\order}{i}{j}{k}
    }
    f(\tplan)
    \geq 
    \nonumber \\
    &~~~\min_{
        \alpha \leq x \leq \beta
    } 
    \left[
    \nijdx{\submodularub}{i}{j}{k}
    \cdot x
    +
    \nijdx{\tailfun}{i}{j}{k}
    (x, D)
    \right]
    \label{eqn:optimal-transport-order-constraint-lb}
\end{eqnarray}
for some suitable $\alpha, \beta$,
where $\langidx$ in~\eqref{eqn:transport-plan-ordering-constraints}, and where 
$
    \nijdx{\submodularub}{i}{j}{k} = \sum_{\ell \in \iset{k}} \ijdx{D}{i_\ell}{j_\ell},
$ 
and
for coefficients $\submodularlb \in \mnReal$,
\begin{equation}
    \nijdx{\tailfun}{i}{j}{k}(x,\submodularlb)
    := 
    \min_{
        \tplan \in \tpoly \intersect 
        \nijdx{\order}{i}{j}{k} :~
        \ijdx{\tplan}{i_\ell}{j_\ell} = x,~\forall \ell \in \iset{k}
    }
    \sum_{
        \ij{p}{q} \in \langidx
    }
    \ijdx{\tplan}{p}{q}
    \ijdx{\submodularlb}{p}{q}.
    \nonumber 
    \label{eqn:optimal-transport-order-constraint-tailfunction}
\end{equation}
We  decouple the row and column constraints in
\eqref{eqn:optimal-transport-order-constraint-tailfunction} and  solve
$\knapsack{
    \{
        \idx{\boundcoef}{i}
    \}_{i \in \iset{n}}, u, \alpha
}$: 
\begin{eqnarray}
    \min_{x \in \Real^n} \sum_{i \in \iset{n}} 
       \idx{\boundcoef}{i}
       \idx{x}{i}
    ~~~\mbox{s.t.}~~~~\sum_{i \in \iset{n}} x_i = \alpha,
    ~~~ 0 \leq x_i \leq u.
    \label{eqn:packing}
\end{eqnarray}
Here $u$ and $\alpha$ in~\eqref{eqn:packing} represent the per item $\idx{x}{i}$ capacity and total budget, respectively.
Define
\begin{eqnarray}
    \knapmu\left(u, \{\boundcoef\}_{i \in \iset{n}}, \alpha \right)
    \!\!\!\!\! &=& \!\!\!\!\!
        \knapsack{
            \{\boundcoef\}_{i \in \iset{n}},
            u, \alpha } 
     \\ 
    \knapnu\left(u, \{\boundcoef\}_{i \in \iset{n-1}}, \alpha \right) 
    \!\!\!\!\! &=& \!\!\!\!\! 
        \knapsack{
            \{\boundcoef\}_{i \in \iset{n-1}},
            u, \alpha - u}.  
  \nonumber  \label{eqn:function-mu-and-nu} 
\end{eqnarray}

\begin{proposition}
    \label{prop:tailfun-bound}
    Let
    $
        \alpha_1 = \max_{i=1}^m \frac{a_i}{n},
        \beta_1 = \max_{i=1}^m a_i,
    $, and
    $
        \alpha_2 = \max_{j=1}^n \frac{b_j}{m},
        \beta_2 = \max_{j=1}^n b_j.
    $
    Then for any $\nij{i}{j}{k}$ set enforcing~\eqref{eqn:optimal-transport-order-constraint2} where $\ni{i}{k}=i_1,i_2,\cdots, i_k$ (and $\ni{j}{k}$) do not repeat row (or column) indices, the quantity $\nijdx{\tailfun}{i}{j}{k} (x, \submodularlb)$ appearing in~\eqref{eqn:optimal-transport-order-constraint-lb}
    is lower-bounded by
    \begin{eqnarray}
        \nijdx{\tailfun}{i}{j}{k} (x, \submodularlb)
        \!\!\! &\geq& \!\!\!
        \left\{
            \begin{array}{cc}
                \tailfunlbRow{ijk}{x,\submodularlb,a}, &  \alpha_1 \leq x \leq \beta_1
                \\
                \tailfunlbCol{ijk}{x,\submodularlb,b}, &   \alpha_2 \leq x \leq \beta_2
            \end{array}
        \right.
        \label{eqn:tailfun-bound}
    \end{eqnarray}
    where $\tailfunlbRow{ijk}{x,\submodularlb,a}$ and $\tailfunlbCol{ijk}{x,\submodularlb,b}$  equal, resp.
    \[
        \sum_{\ell \in \iset{k}}
        \knapnu
        \left(x, 
            \{
                \idx{\boundcoef}{\ij{i_\ell}{q}}
            \}_{q \in \iset{n} \setminus \{j_\ell\}},
            a_{i_\ell}
        \right)
        + 
        \sum_{p \notin \ni{i}{k}} 
        \knapmu
        \left( 
            x, 
            \{
                \idx{\boundcoef}{\ij{p}{q}}
            \}_{q \in \iset{n}}, 
            a_p 
        \right)
    \]
   
    \[
        \sum_{\ell \in \iset{k}}
        \knapnu
        \left(
            x, 
            \{
                \idx{\boundcoef}{\ij{p}{j_\ell}}
            \}_{p  \in \iset{m} \setminus \{i_\ell\}}, 
            b_{j_\ell}
        \right)
        %
        +
        \sum_{q \notin \ni{j}{k}} 
        \knapmu
        \left(
            x, \{
                \idx{\boundcoef}{\ij{p}{q}}
            \}_{p \in \iset{m}}, 
            b_q
        \right)
    \]
\end{proposition}

The proof is in the Appendix.
Set $\submodularlb = D$ in~\eqref{eqn:tailfun-bound}. Now, from \eqref{eqn:optimal-transport-order-constraint-lb} and Prop.~\ref{prop:tailfun-bound} 
we obtain: 
\begin{eqnarray}
    &&\min_{
        \tplan \in \tpoly \intersect 
        \nijdx{\order}{i}{j}{k}
    }
    f(\tplan)
    \nonumber
   \\
    &\geq&
    \max 
    \left(
        \begin{array}{c}
            \min_{
                \alpha_1 \leq x \leq \beta_1
            }
            \nijdx{\submodularub}{i}{j}{k} \cdot
            x 
            + 
            \tailfunlbRow{ijk}{x,D,a}
            \\
            \min_{
                \alpha_2 \leq x \leq \beta_2
            }
            \nijdx{\submodularub}{i}{j}{k} \cdot
            x 
            + 
            \tailfunlbCol{ijk}{x,D,b}
        \end{array}
    \right)
    \nonumber 
    \\
    \label{eqn:final-bound}
\end{eqnarray}
Thus \eqref{eqn:final-bound} provides  the lower bound  in Line 4 of Alg.~\ref{alg:heuristic-search}.

\section{Experimental Results}
\label{sec:experiments}

{
\begin{table}
    \begin{center}
        \caption{
            Task and annotation scores  on e-SNLI.  \texttt{BestF1@n} is the best achievable annotation \texttt{F1} score over the set of $k_2=$ \texttt{n}  best plans, measured by setting coefficients  0 or 1 after determining  values under which there is no impact on the  the  task \texttt{F1}  (75.1 $\pm$ .3). Standard OT  achieves 64.5 $\pm$ .3 annotation \texttt{F1}. 
        }
        \label{tab:esnli2}
        \begin{tabular}{|l|c|c|c|} 
        \hline
        \hline
            \multirow{2}{*}{Algorithm} &  \multicolumn{3}{c|}{\texttt{BestF1@n}} \\ 
            \cline{2-4}
             & \texttt{n}$=2$ & \texttt{n}$=5$ & \texttt{n}$=10$ \\
            \hline 
            Algorithm~\ref{alg:heuristic-search} (ours) & 68.1 $\pm$ .2 & 71.2 $\pm$ .3 & 73.7 $\pm$ .2 \\
            \hline
            Greedy version & 67.9 $\pm$ .3 & 68.2 $\pm$ .3 & 68.2 $\pm$ .3 \\
           \hline
           \hline
        \end{tabular}
    \end{center}
\end{table}
}
\begin{figure*}
    \centering
    \includegraphics[width=.75\textwidth]{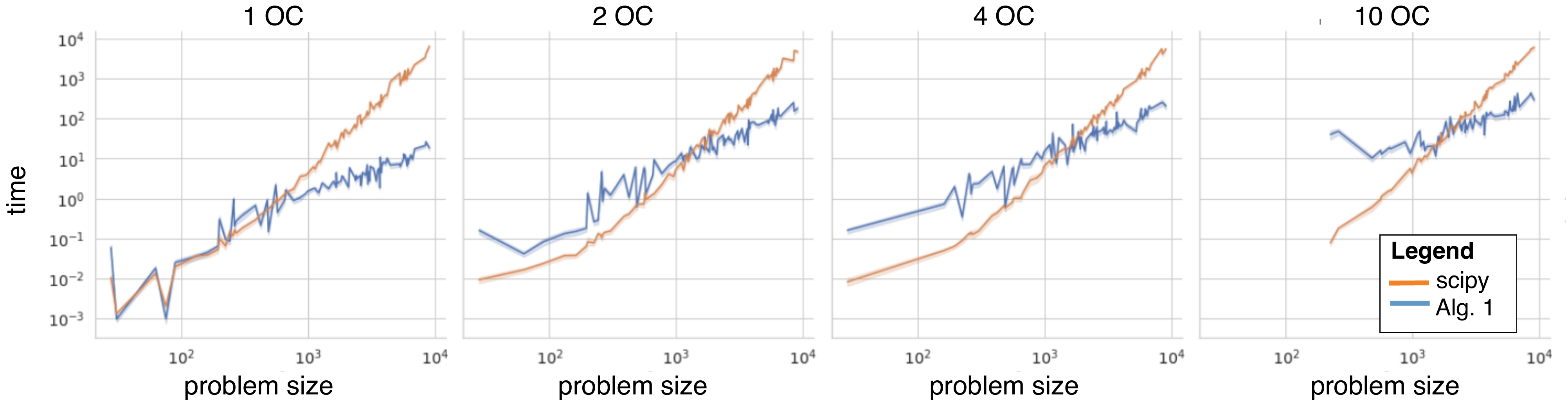}
    \caption{Compute time of Alg.~\ref{alg:admm} compared with \texttt{scipy.optimize}, for various number of constraints.}
    \label{fig:computation}
\end{figure*}

\paragraph{Sentence Relationship Classification:}
We use  an annotated dataset from the \emph{enhanced Stanford Natural Language Inference} (e-SNLI)  \citep{camburu_e-snli_2018,swanson_rationalizing_2020}
that includes English sentence pairs classified as  ``entailment'', ``contradiction'' or ``neutral''.  Annotations  denote which words were used by humans to determine the class.

We use Alg.~\ref{alg:heuristic-search} to compute $k_2$ candidate plans which are  measured against the annotations after being set  to a binary variable:  values above a threshold are set to 1, and below to 0. 
The threshold is determined as the one that maintains the reported task \texttt{F1} score when the  plan weights that lie below the threshold are set to zero. 
We determine the $(\tau_1,\tau_2)$ that constrain $\searchtree(k_3,\tau_1,\tau_2)$ to be $(\tau_1,\tau_2) = (.5,.5)$.
Annotations are provided separately on the source and target  in each pair; each candidate plan is marginalized using \texttt{max} across columns/rows to arrive at a vector of coefficients, and used to compute an annotation \texttt{F1} score against the annotations (after application of the threshold).
The top $k_2 = $ \texttt{n} plans are scored using the \texttt{BestF1@n}  metric, which reports the score of the best plan in the learnt subtree $\esttree(k_1,k_2,k_3,\tau_1,\tau_2)$. More details can be found in the Appendix.

Table~\ref{tab:esnli2} shows the  results  in terms of the  classification task and annotation accuracy along with confidence intervals.  Using even only  a single order constraint, i.e. $k_3=1$, the tree search in Alg.~\ref{alg:heuristic-search} achieves significantly improved explainability using a modest $k_1 = 20$ candidates, giving higher annotation scores over the set of plans. 
As a baseline against which to compare, we implement a greedy algorithm  by  restricting the set $\variates(\tplan)$ of variates in \eqref{eqn:heuristic-statistic} to  contain a single  variate, i.e., modifying  Lines 2 and 11 of Alg. \ref{alg:heuristic-search}. 
The greedy algorithm retains the   variate $\ij{i}{j} \in \variates(\tplan)$ with the lowest saturation~\eqref{eqn:heuristic-statistic-triple}, thus  corresponding to a single tree path. 
For $n \geq 5$, the greedy algorithm results in a candidate set  lacking diversity, as evidenced by the stagnating \texttt{BestF1@n} scores as compared to Alg.~\ref{alg:heuristic-search}.

Fig.~\ref{fig:multi-oc-search-tree} shows the search tree $\searchtree(k_3,\tau_1,\tau_2)$ (blue), the learnt subtree $\esttree(k_1,k_2,k_3,\tau_1,\tau_2)$ (red) and the corresponding  candidate plans 
for the  example of Fig.~\ref{fig:intro-example} for the top $k_2=5$ candidates and  depth increased to $k_3=3$.
See Supp. Mat. for a walk-through of Alg.~\ref{alg:heuristic-search} using  Fig.~\ref{fig:multi-oc-search-tree} and a study of the
effectiveness of the lower bound using a  ``swarmplot'' of  \# of nodes skipped in the search tree, $\searchtree(k_3, \tau_1, \tau_2)$ as tree depth $k_3$ increases.
The numbers in Fig.~\ref{fig:multi-oc-search-tree} indicate the ranking in the top-$k_2$ order upon termination.
 The diversity provided by Alg.~\ref{alg:heuristic-search} allows  the 3rd best solution (in red) to be included in the set of plans, and it is that plan that best explains the contradiction in the two sentences.

\paragraph{Color Transfer}
We use source  images from the SUN dataset~\citep{xiao_sun_2010,yu_lsun_2016}, and  target images from   WikiArt \citep{tan_ceci_2016}.
The  images are segmented and the average RGB colors in each segment are used to obtain  distributions $a,b$ and  costs $D$ in the OT formulations \eqref{eqn:optimal-transport} and \eqref{eqn:obj2}. 
The thresholds that constrain $\searchtree(k_3,\tau_1,\tau_2)$ are set to $(\tau_1,\tau_2) = (.5,1.)$ here. Further details on the problem setup can be found in the Appendix.

Results are shown in Fig.~\ref{fig:expr-color-transfer}. The first OT with OC plan includes a prior (human-crafted) constraint while the latter two provide the auto-generated constraints from  Alg. \ref{alg:heuristic-search}.
The
 OT with OC solutions offer a  range of  images, all of which improve upon the standard OT solution in diverse ways. 
In the first row, standard OT  transfers the blue sky to the ground and the cylindrical roof, while the prior constraint maps the green grass from the painting to the ground and intensifies the sky with the blue. The auto-generated constraints from Alg. \ref{alg:heuristic-search}, especially the $k_3=2$ case, do the same though slightly less.
In the second row, standard OT  awkwardly transferred the red to the grass path. 
The OT with OC solution with a prior constraint as well as the auto-generated constraints of Alg. \ref{alg:heuristic-search}, $k_3=2$   use the red to create an evening sky effect.
In the third row,  standard OT  results in a very dark cliff with minimal visible features; the OT with OC solutions lead to a well-defined cliff  and more effectively use the blue from the painting sky to deepen the sky in the source image.
In addition to producing better resulting images,  Alg. \ref{alg:heuristic-search} allows for human judgement to be used to make the final selection from a set of plans.

\paragraph{Computational Efficiency}

Fig.~\ref{fig:computation} shows how the computation time (in secs) of  Alg.~\ref{alg:admm} scales with problem size. We generate 100 random problems for  $m,n \leq 100 $  with 1, 2, 4 and 10 order constraints. 
The iterations are set to terminate at 1e4 rounds or a max projection error of 1e-4,
and these settings achieve an average functional approximation of 0.51\% error (within $\pm .19$).
We use penalty $\rho = 1.0$ after testing a range of $\rho$'s and observing little difference, as it is well-known in the ADMM literature~\citep{boyd_distributed_2010}.
We compare our method against \texttt{scipy.optimize}.

Fig.~\ref{fig:computation} shows 95\% confidence intervals from 10 independent repetitions.   Alg.~\ref{alg:admm} 
scales much better than \texttt{scipy.optimize} for large problems by orders of magnitude as the problem size $mn$ increases.
Note that Fig.~\ref{fig:computation}  compares \texttt{python}-based algorithms; we also evaluated the \texttt{C++}-based \texttt{cvxpy}, and found that our algorithm performs almost indistinguishably to \texttt{cvxpy} for a single order constraint, which is remarkable given the overhead of \texttt{python} as compared to \texttt{C++}.



\section{Conclusions}
Our proposed optimal transport with order constraints allows  complex structure to be incorporated in  optimal transport plans and provides a set of explainable solutions from which a human can  select. Future extensions of interest are as follows. Other OT solvers may be extended to incorporate  order constraints. Such solvers may offer different sets of plans satisfying varying objectives arising from other loss functions and various forms of regularization. See \cite{flamary2021pot}. Nonlinear costs such as the submodular formulation of \cite{alvarez-melis_structured_2017} can be useful in some settings and may be extended to include order constraints. 
It would also be of interest to explore other applications that can benefit from OT and specifically OT with OC.
Optimal Transport with Order Constraints can be found in the AI Explainability 360 toolbox, which is part of the IBM Research Trusted AI library ~\cite{aix360}  at \texttt{https://github.com/Trusted-AI/AIX360}.



\bibliography{references}

\begin{thebibliography}{47}
\providecommand{\natexlab}[1]{#1}
\providecommand{\url}[1]{\texttt{#1}}
\expandafter\ifx\csname urlstyle\endcsname\relax
  \providecommand{\doi}[1]{doi: #1}\else
  \providecommand{\doi}{doi: \begingroup \urlstyle{rm}\Url}\fi

\bibitem[Altschuler et~al.(2017)Altschuler, Weed, and
  Rigollet]{altschuler_near-linear_2017}
Altschuler, J., Weed, J., and Rigollet, P.
\newblock Near-linear time approximation algorithms for optimal transport via
  sinkhorn iteration.
\newblock In \emph{Proceedings of the 31st International Conference on Neural
  Information Processing Systems}, NIPS'17, pp.\  1961–1971, Red Hook, NY,
  USA, 2017. Curran Associates Inc.
\newblock ISBN 9781510860964.

\bibitem[Altschuler et~al.(2019)Altschuler, Bach, Rudi, and
  Niles-Weed]{altschuler_massively_2019}
Altschuler, J., Bach, F., Rudi, A., and Niles-Weed, J.
\newblock Massively scalable sinkhorn distances via the nystr\"{o}m method.
\newblock In Wallach, H., Larochelle, H., Beygelzimer, A., d\textquotesingle
  Alch\'{e}-Buc, F., Fox, E., and Garnett, R. (eds.), \emph{Advances in Neural
  Information Processing Systems}, volume~32. Curran Associates, Inc., 2019.

\bibitem[Altschuler \& Boix-Adsera(2020{\natexlab{a}})Altschuler and
  Boix-Adsera]{altschuler_hardness_2020}
Altschuler, J.~M. and Boix-Adsera, E.
\newblock Hardness results for {Multimarginal} {Optimal} {Transport} problems.
\newblock \emph{arXiv:2012.05398 [cs, math]}, December 2020{\natexlab{a}}.
\newblock arXiv: 2012.05398.

\bibitem[Altschuler \& Boix-Adsera(2020{\natexlab{b}})Altschuler and
  Boix-Adsera]{altschuler_polynomial-time_2020}
Altschuler, J.~M. and Boix-Adsera, E.
\newblock Polynomial-time algorithms for multimarginal optimal transport
  problems with structure.
\newblock 2020{\natexlab{b}}.
\newblock URL \url{http://arxiv.org/abs/2008.03006}.

\bibitem[Alvarez-Melis et~al.(2018)Alvarez-Melis, Jaakkola, and
  Jegelka]{alvarez-melis_structured_2017}
Alvarez-Melis, D., Jaakkola, T., and Jegelka, S.
\newblock Structured optimal transport.
\newblock In Storkey, A. and Perez-Cruz, F. (eds.), \emph{Proceedings of the
  Twenty-First International Conference on Artificial Intelligence and
  Statistics}, volume~84 of \emph{Proceedings of Machine Learning Research},
  pp.\  1771--1780. PMLR, 09--11 Apr 2018.
\newblock URL \url{http://proceedings.mlr.press/v84/alvarez-melis18a.html}.

\bibitem[Alvarez-Melis et~al.(2019)Alvarez-Melis, Jegelka, and
  Jaakkola]{alvarez-melis_towards_2019}
Alvarez-Melis, D., Jegelka, S., and Jaakkola, T.~S.
\newblock Towards optimal transport with global invariances.
\newblock In Chaudhuri, K. and Sugiyama, M. (eds.), \emph{Proceedings of
  Machine Learning Research}, volume~89 of \emph{Proceedings of Machine
  Learning Research}, pp.\  1870--1879. PMLR, 16--18 Apr 2019.
\newblock URL \url{http://proceedings.mlr.press/v89/alvarez-melis19a.html}.

\bibitem[Alvarez-Melis et~al.(2020)Alvarez-Melis, Mroueh, and
  Jaakkola]{hierarchy}
Alvarez-Melis, D., Mroueh, Y., and Jaakkola, T.
\newblock Unsupervised hierarchy matching with optimal transport over
  hyperbolic spaces.
\newblock In \emph{International Conference on Artificial Intelligence and
  Statistics}, pp.\  1606--1617. PMLR, 2020.

\bibitem[Arya et~al.(2019)Arya, Bellamy, Chen, Dhurandhar, Hind, Hoffman,
  Houde, Liao, Luss, Mojsilovi\'c, Mourad, Pedemonte, Raghavendra, Richards,
  Sattigeri, Shanmugam, Singh, Varshney, Wei, and Zhang]{aix360}
Arya, V., Bellamy, R. K.~E., Chen, P.-Y., Dhurandhar, A., Hind, M., Hoffman,
  S.~C., Houde, S., Liao, Q.~V., Luss, R., Mojsilovi\'c, A., Mourad, S.,
  Pedemonte, P., Raghavendra, R., Richards, J., Sattigeri, P., Shanmugam, K.,
  Singh, M., Varshney, K.~R., Wei, D., and Zhang, Y.
\newblock One explanation does not fit all: A toolkit and taxonomy of ai
  explainability techniques, 2019.
\newblock URL \url{https://arxiv.org/abs/1909.03012}.

\bibitem[Beck(2017)]{beck_first-order_2017}
Beck, A.
\newblock \emph{First-{Order} {Methods} in {Optimization}}.
\newblock Society for Industrial and Applied Mathematics, Philadelphia, PA,
  October 2017.
\newblock ISBN 978-1-61197-498-0 978-1-61197-499-7.
\newblock \doi{10.1137/1.9781611974997}.
\newblock URL \url{http://epubs.siam.org/doi/book/10.1137/1.9781611974997}.

\bibitem[Blondel et~al.(2018)Blondel, Seguy, and Rolet]{blondel_smooth_2017}
Blondel, M., Seguy, V., and Rolet, A.
\newblock Smooth and sparse optimal transport.
\newblock In Storkey, A. and Perez-Cruz, F. (eds.), \emph{Proceedings of the
  Twenty-First International Conference on Artificial Intelligence and
  Statistics}, volume~84 of \emph{Proceedings of Machine Learning Research},
  pp.\  880--889. PMLR, 09--11 Apr 2018.
\newblock URL \url{http://proceedings.mlr.press/v84/blondel18a.html}.

\bibitem[Boyd(2010)]{boyd_distributed_2010}
Boyd, S.
\newblock Distributed {Optimization} and {Statistical} {Learning} via the
  {Alternating} {Direction} {Method} of {Multipliers}.
\newblock \emph{Foundations and Trends® in Machine Learning}, 3\penalty0
  (1):\penalty0 1--122, 2010.
\newblock ISSN 1935-8237, 1935-8245.
\newblock \doi{10.1561/2200000016}.
\newblock URL \url{http://www.nowpublishers.com/article/Details/MAL-016}.

\bibitem[Boyd et~al.(2011)Boyd, Parikh, Chu, Peleato, and Eckstein]{admm}
Boyd, S., Parikh, N., Chu, E., Peleato, B., and Eckstein, J.
\newblock Distributed optimization and statistical learning via the alternating
  direction method of multipliers.
\newblock \emph{Found. Trends Mach. Learn.}, 3\penalty0 (1):\penalty0 1–122,
  jan 2011.
\newblock ISSN 1935-8237.
\newblock \doi{10.1561/2200000016}.
\newblock URL \url{https://doi.org/10.1561/2200000016}.

\bibitem[Boyd \& Vandenberghe(2004)Boyd and Vandenberghe]{boyd_convex_2004}
Boyd, S.~P. and Vandenberghe, L.
\newblock \emph{Convex optimization}.
\newblock Cambridge University Press, Cambridge, UK ; New York, 2004.
\newblock ISBN 978-0-521-83378-3.

\bibitem[Camburu et~al.(2018)Camburu, Rockt\"{a}schel, Lukasiewicz, and
  Blunsom]{camburu_e-snli_2018}
Camburu, O.-M., Rockt\"{a}schel, T., Lukasiewicz, T., and Blunsom, P.
\newblock e-{SNLI}: Natural language inference with natural language
  explanations.
\newblock In Bengio, S., Wallach, H., Larochelle, H., Grauman, K.,
  Cesa-Bianchi, N., and Garnett, R. (eds.), \emph{Advances in Neural
  Information Processing Systems}, volume~31. Curran Associates, Inc., 2018.

\bibitem[{Courty} et~al.(2017){Courty}, {Flamary}, {Tuia}, and
  {Rakotomamonjy}]{courty_optimal_2016}
{Courty}, N., {Flamary}, R., {Tuia}, D., and {Rakotomamonjy}, A.
\newblock Optimal transport for domain adaptation.
\newblock \emph{IEEE Transactions on Pattern Analysis and Machine
  Intelligence}, 39\penalty0 (9):\penalty0 1853--1865, 2017.
\newblock \doi{10.1109/TPAMI.2016.2615921}.

\bibitem[Cuturi(2013)]{cuturi_sinkhorn_2013}
Cuturi, M.
\newblock Sinkhorn distances: Lightspeed computation of optimal transport.
\newblock In Burges, C. J.~C., Bottou, L., Welling, M., Ghahramani, Z., and
  Weinberger, K.~Q. (eds.), \emph{Advances in Neural Information Processing
  Systems 26}, pp.\  2292--2300. Curran Associates, Inc., 2013.

\bibitem[De~Leeuw et~al.(2009)De~Leeuw, Kurt, and Mair]{de_leeuw_isotone_2009}
De~Leeuw, J., Kurt, H., and Mair, P.
\newblock Isotone {Optimization} in {R}: {Pool}-{Adjacent}-{Violators}
  {Algorithm} ({PAVA}) and {Active} {Set} {Methods}.
\newblock \emph{Journal of Statistical Software}, 32, October 2009.
\newblock \doi{10.18637/jss.v032.i05}.

\bibitem[Di~Marino et~al.(2015)Di~Marino, Gerolin, and
  Nenna]{di_marino_optimal_2015}
Di~Marino, S., Gerolin, A., and Nenna, L.
\newblock Optimal transportation theory with repulsive costs.
\newblock 2015.
\newblock URL \url{http://arxiv.org/abs/1506.04565}.

\bibitem[Felzenszwalb \& Huttenlocher(2004)Felzenszwalb and
  Huttenlocher]{felzenszwalb_efficient_2004}
Felzenszwalb, P.~F. and Huttenlocher, D.~P.
\newblock Efficient {Graph}-{Based} {Image} {Segmentation}.
\newblock \emph{International Journal of Computer Vision}, 59\penalty0
  (2):\penalty0 167--181, September 2004.
\newblock ISSN 0920-5691.
\newblock \doi{10.1023/B:VISI.0000022288.19776.77}.
\newblock URL
  \url{http://link.springer.com/10.1023/B:VISI.0000022288.19776.77}.

\bibitem[Flamary et~al.(2021)Flamary, Courty, Gramfort, Alaya, Boisbunon,
  Chambon, Chapel, Corenflos, Fatras, Fournier, Gautheron, Gayraud, Janati,
  Rakotomamonjy, Redko, Rolet, Schutz, Seguy, Sutherland, Tavenard, Tong, and
  Vayer]{flamary2021pot}
Flamary, R., Courty, N., Gramfort, A., Alaya, M.~Z., Boisbunon, A., Chambon,
  S., Chapel, L., Corenflos, A., Fatras, K., Fournier, N., Gautheron, L.,
  Gayraud, N.~T., Janati, H., Rakotomamonjy, A., Redko, I., Rolet, A., Schutz,
  A., Seguy, V., Sutherland, D.~J., Tavenard, R., Tong, A., and Vayer, T.
\newblock Pot: Python optimal transport.
\newblock \emph{Journal of Machine Learning Research}, 22\penalty0
  (78):\penalty0 1--8, 2021.
\newblock URL \url{http://jmlr.org/papers/v22/20-451.html}.

\bibitem[Forrow et~al.(2019)Forrow, H{\"u}tter, Nitzan, Rigollet, Schiebinger,
  and Weed]{forrow}
Forrow, A., H{\"u}tter, J.-C., Nitzan, M., Rigollet, P., Schiebinger, G., and
  Weed, J.
\newblock Statistical optimal transport via factored couplings.
\newblock In \emph{The 22nd International Conference on Artificial Intelligence
  and Statistics}, pp.\  2454--2465. PMLR, 2019.

\bibitem[Goldfeld \& Greenewald(2020)Goldfeld and Greenewald]{goldfeld}
Goldfeld, Z. and Greenewald, K.
\newblock Gaussian-smoothed optimal transport: Metric structure and statistical
  efficiency.
\newblock In \emph{International Conference on Artificial Intelligence and
  Statistics}, pp.\  3327--3337. PMLR, 2020.

\bibitem[Grotzinger \& Witzgall(1984)Grotzinger and
  Witzgall]{grotzinger_projections_1984}
Grotzinger, S.~J. and Witzgall, C.
\newblock Projections onto order simplexes.
\newblock \emph{Applied Mathematics and Optimization}, 12\penalty0
  (1):\penalty0 247--270, October 1984.
\newblock ISSN 1432-0606.
\newblock \doi{10.1007/BF01449044}.

\bibitem[Guminov et~al.(2021)Guminov, Dvurechensky, Tupitsa, and
  Gasnikov]{guminov_combination_2021}
Guminov, S., Dvurechensky, P., Tupitsa, N., and Gasnikov, A.
\newblock On a {Combination} of {Alternating} {Minimization} and {Nesterov}’s
  {Momentum}.
\newblock In \emph{Proceedings of the 38th {International} {Conference} on
  {Machine} {Learning}}, pp.\  3886--3898. PMLR, July 2021.
\newblock URL \url{https://proceedings.mlr.press/v139/guminov21a.html}.
\newblock ISSN: 2640-3498.

\bibitem[Guo et~al.(2020)Guo, Ho, and Jordan]{guo2020}
Guo, W., Ho, N., and Jordan, M.
\newblock Fast algorithms for computational optimal transport and wasserstein
  barycenter.
\newblock In \emph{International Conference on Artificial Intelligence and
  Statistics}, pp.\  2088--2097. PMLR, 2020.

\bibitem[Jambulapati et~al.(2019)Jambulapati, Sidford, and
  Tian]{jambulapati_direct_2019}
Jambulapati, A., Sidford, A., and Tian, K.
\newblock A {Direct} o(1/epsilon) {Iteration} {Parallel} {Algorithm} for
  {Optimal} {Transport}.
\newblock In \emph{Advances in {Neural} {Information} {Processing} {Systems}},
  volume~32. Curran Associates, Inc., 2019.

\bibitem[Kusner et~al.(2015)Kusner, Sun, Kolkin, and
  Weinberger]{kusner_word_2015}
Kusner, M., Sun, Y., Kolkin, N., and Weinberger, K.
\newblock From word embeddings to document distances.
\newblock In Bach, F. and Blei, D. (eds.), \emph{Proceedings of the 32nd
  International Conference on Machine Learning}, volume~37 of \emph{Proceedings
  of Machine Learning Research}, pp.\  957--966, Lille, France, 07--09 Jul
  2015. PMLR.
\newblock URL \url{http://proceedings.mlr.press/v37/kusnerb15.html}.

\bibitem[Laclau et~al.(2021)Laclau, Redko, Choudhary, and Largeron]{laclau}
Laclau, C., Redko, I., Choudhary, M., and Largeron, C.
\newblock All of the fairness for edge prediction with optimal transport.
\newblock In \emph{International Conference on Artificial Intelligence and
  Statistics}, pp.\  1774--1782. PMLR, 2021.

\bibitem[Lin et~al.(2019)Lin, Ho, and Jordan]{lin2019a}
Lin, T., Ho, N., and Jordan, M.
\newblock On efficient optimal transport: An analysis of greedy and accelerated
  mirror descent algorithms.
\newblock In \emph{International Conference on Machine Learning}, pp.\
  3982--3991. PMLR, 2019.

\bibitem[Liu et~al.(2021)Liu, Rao, Lu, Zhou, and Hsieh]{imageAAAI2021}
Liu, B., Rao, Y., Lu, J., Zhou, J., and Hsieh, C.-J.
\newblock Multi-proxy wasserstein classifier for image classification.
\newblock \emph{Proceedings of the AAAI Conference on Artificial Intelligence},
  35\penalty0 (10):\penalty0 8618--8626, May 2021.
\newblock URL \url{https://ojs.aaai.org/index.php/AAAI/article/view/17045}.

\bibitem[Pass(2015)]{pass_multi-marginal_2014}
Pass, B.
\newblock Multi-marginal optimal transport: Theory and applications.
\newblock \emph{ESAIM: M2AN}, 49\penalty0 (6):\penalty0 1771--1790, 2015.
\newblock \doi{10.1051/m2an/2015020}.
\newblock URL \url{https://doi.org/10.1051/m2an/2015020}.

\bibitem[Paty et~al.(2020)Paty, d’Aspremont, and Cuturi]{regularity}
Paty, F.-P., d’Aspremont, A., and Cuturi, M.
\newblock Regularity as regularization: Smooth and strongly convex brenier
  potentials in optimal transport.
\newblock In \emph{International Conference on Artificial Intelligence and
  Statistics}, pp.\  1222--1232. PMLR, 2020.

\bibitem[Peyré \& Cuturi(2019)Peyré and Cuturi]{peyre_computational_2020}
Peyré, G. and Cuturi, M.
\newblock Computational optimal transport: With applications to data science.
\newblock \emph{Foundations and Trends® in Machine Learning}, 11\penalty0
  (5-6):\penalty0 355--607, 2019.
\newblock ISSN 1935-8237.
\newblock \doi{10.1561/2200000073}.
\newblock URL \url{http://dx.doi.org/10.1561/2200000073}.

\bibitem[Scetbon et~al.(2021)Scetbon, Cuturi, and
  Peyré]{scetbon_low-rank_2021}
Scetbon, M., Cuturi, M., and Peyré, G.
\newblock Low-{Rank} {Sinkhorn} {Factorization}.
\newblock In \emph{Proceedings of the 38th {International} {Conference} on
  {Machine} {Learning}}, pp.\  9344--9354. PMLR, July 2021.
\newblock URL \url{https://proceedings.mlr.press/v139/scetbon21a.html}.
\newblock ISSN: 2640-3498.

\bibitem[Schmitzer(2019)]{schmitzer_stabilized_2019}
Schmitzer, B.
\newblock Stabilized sparse scaling algorithms for entropy regularized
  transport problems.
\newblock \emph{SIAM Journal on Scientific Computing}, 41\penalty0
  (3):\penalty0 A1443--A1481, 2019.
\newblock \doi{10.1137/16M1106018}.

\bibitem[Shi et~al.(2020)Shi, Yu, Liu, Zhang, and Li]{imageAAAI2020}
Shi, Y., Yu, X., Liu, L., Zhang, T., and Li, H.
\newblock Optimal feature transport for cross-view image geo-localization.
\newblock \emph{Proceedings of the AAAI Conference on Artificial Intelligence},
  34:\penalty0 11990--11997, Apr. 2020.
\newblock \doi{10.1609/aaai.v34i07.6875}.
\newblock URL \url{https://ojs.aaai.org/index.php/AAAI/article/view/6875}.

\bibitem[Solomon(2018)]{solomon_optimal_2018}
Solomon, J.
\newblock Optimal transport on discrete domains.
\newblock 2018.
\newblock URL \url{http://arxiv.org/abs/1801.07745}.

\bibitem[Su \& Hua(2017)Su and Hua]{DTW}
Su, B. and Hua, G.
\newblock Order-preserving wasserstein distance for sequence matching.
\newblock In \emph{2017 IEEE Conference on Computer Vision and Pattern
  Recognition (CVPR)}, pp.\  2906--2914, 2017.
\newblock \doi{10.1109/CVPR.2017.310}.

\bibitem[Swanson et~al.(2020)Swanson, Yu, and Lei]{swanson_rationalizing_2020}
Swanson, K., Yu, L., and Lei, T.
\newblock Rationalizing text matching: Learning sparse alignments via optimal
  transport.
\newblock In \emph{Proceedings of the 58th Annual Meeting of the Association
  for Computational Linguistics}, pp.\  5609--5626. Association for
  Computational Linguistics, 2020.
\newblock \doi{10.18653/v1/2020.acl-main.496}.
\newblock URL \url{https://www.aclweb.org/anthology/2020.acl-main.496}.

\bibitem[Tan et~al.(2016)Tan, Chan, Aguirre, and Tanaka]{tan_ceci_2016}
Tan, W.~R., Chan, C.~S., Aguirre, H.~E., and Tanaka, K.
\newblock Ceci n'est pas une pipe: {A} deep convolutional network for fine-art
  paintings classification.
\newblock In \emph{2016 {IEEE} {International} {Conference} on {Image}
  {Processing} ({ICIP})}, pp.\  3703--3707, Phoenix, AZ, USA, September 2016.
  IEEE.
\newblock ISBN 978-1-4673-9961-6.
\newblock \doi{10.1109/ICIP.2016.7533051}.
\newblock URL \url{http://ieeexplore.ieee.org/document/7533051/}.

\bibitem[Villani(2008)]{villani}
Villani, C.
\newblock \emph{Optimal transport: old and new}, volume 338.
\newblock Springer Science \& Business Media, 2008.

\bibitem[{Wang} et~al.(2010){Wang}, {Li}, and {Konig}]{wang_learning_2010}
{Wang}, F., {Li}, P., and {Konig}, A.~C.
\newblock Learning a bi-stochastic data similarity matrix.
\newblock In \emph{2010 IEEE International Conference on Data Mining}, pp.\
  551--560, 2010.
\newblock \doi{10.1109/ICDM.2010.141}.

\bibitem[Wu et~al.(2018)Wu, Yen, Xu, Xu, Balakrishnan, Chen, Ravikumar, and
  Witbrock]{wu_word_2018}
Wu, L., Yen, I. E.-H., Xu, K., Xu, F., Balakrishnan, A., Chen, P.-Y.,
  Ravikumar, P., and Witbrock, M.~J.
\newblock Word mover{'}s embedding: From {W}ord2{V}ec to document embedding.
\newblock In \emph{Proceedings of the 2018 Conference on Empirical Methods in
  Natural Language Processing}, pp.\  4524--4534, Brussels, Belgium,
  October-November 2018. Association for Computational Linguistics.
\newblock \doi{10.18653/v1/D18-1482}.
\newblock URL \url{https://www.aclweb.org/anthology/D18-1482}.

\bibitem[Xiao et~al.(2010)Xiao, Hays, Ehinger, Oliva, and
  Torralba]{xiao_sun_2010}
Xiao, J., Hays, J., Ehinger, K.~A., Oliva, A., and Torralba, A.
\newblock {SUN} database: {Large}-scale scene recognition from abbey to zoo.
\newblock In \emph{2010 {IEEE} {Computer} {Society} {Conference} on {Computer}
  {Vision} and {Pattern} {Recognition}}, pp.\  3485--3492, June 2010.
\newblock \doi{10.1109/CVPR.2010.5539970}.
\newblock ISSN: 1063-6919.

\bibitem[Yu et~al.(2016)Yu, Seff, Zhang, Song, Funkhouser, and
  Xiao]{yu_lsun_2016}
Yu, F., Seff, A., Zhang, Y., Song, S., Funkhouser, T., and Xiao, J.
\newblock {LSUN}: {Construction} of a {Large}-scale {Image} {Dataset} using
  {Deep} {Learning} with {Humans} in the {Loop}.
\newblock \emph{arXiv:1506.03365 [cs]}, June 2016.
\newblock URL \url{http://arxiv.org/abs/1506.03365}.
\newblock arXiv: 1506.03365.

\bibitem[Yurochkin et~al.(2019)Yurochkin, Claici, Chien, Mirzazadeh, and
  Solomon]{yurochkin_hierarchical_2019}
Yurochkin, M., Claici, S., Chien, E., Mirzazadeh, F., and Solomon, J.~M.
\newblock Hierarchical optimal transport for document representation.
\newblock In Wallach, H., Larochelle, H., Beygelzimer, A., Alché-Buc, F.~d.,
  Fox, E., and Garnett, R. (eds.), \emph{Advances in Neural Information
  Processing Systems 32}, pp.\  1601--1611. Curran Associates, Inc., 2019.

\bibitem[Zhang et~al.(2021)Zhang, Cheng, and Reeves]{zhang2021convergence}
Zhang, Y., Cheng, X., and Reeves, G.
\newblock Convergence of gaussian-smoothed optimal transport distance with
  sub-gamma distributions and dependent samples.
\newblock In \emph{International Conference on Artificial Intelligence and
  Statistics}, pp.\  2422--2430. PMLR, 2021.

\end{thebibliography}
\bibliographystyle{icml2022}

\iftrue
\newpage
\appendix
\onecolumn

\section*{Appendix}

The Appendix comprises two sections.  Section A includes further details on  the Sentence Relationship Classification and Color Transfer experiments and  on the Computational Efficiency of our proposed OT with OC method. Section B of the Appendix includes the proofs for the results in the main paper.

\section{Details on Experimental Setup and Results}

\subsection{Sentence Relationship Classification}

For the \emph{enhanced Stanford Natural Language Inference} (e-SNLI) dataset we follow~\citet{swanson_rationalizing_2020} to evaluate explainablity of an OT scheme, whereby annotation-based scores are measured after applying a threshold to determine the activations.
The threshold is taken as the one that balances task and explainabilty performance; it should be a sufficiently high so that we do not get spurious activations, and sufficiently low so that enough coefficients are preserved to not compromise the task scores.
e-SNLI provides annotation labels  to determine the explainability performance, quantified by the annotation \texttt{F1}. The highest annotation \texttt{F1} score over \texttt{n} $=k_2$  candidates of the Algorithm~\ref{alg:heuristic-search} determines the \texttt{BestF1@n} scores. The Greedy version used as a baseline is scored the same way.
We used sizes of (100K, 10K, 5K) for train, validation, and test, respectively. 

\paragraph{Picking the threshold for the  classification task:} 
This is a three class ``entailment'', ``neutral'' and ``contradiction'' classification task to predict logical relationships between sentence pairs. 
The task is performed by a shallow neural network that incorporates an OT attention module between the input and output layers layers\footnote{\texttt{https://github.com/asappresearch/rationale-alignment}}.
The attention layer uses OT to match source and target tokens, and matched tokens are concatenated and output via softmax.
The shallow network is trained on the un-thresholded plan. Over the validation and tests sets, all coefficients that lie below $T / (mn)$ are dropped when measuring the task \texttt{F1}.
Using the validation set, we pick the threshold in the region between $0.01$ and $2$ where the task \texttt{F1} starts to plateau due to the dropped coefficients; we use an activation threshold of $2$.

\paragraph{Comparing activations with annotation labels:} In e-SNLI, annotations that are provided separately on the source and target sentence. \citet{swanson_rationalizing_2020} proposed to marginalize the transport plan $\tplan \in \tpoly$ forming two vectors for comparison, by applying \texttt{max} to each row and column;
we found that this method overlooks the coupling in the transport plan $\tplan$, and we propose instead taking $\max$ only over row/column positions that have annotations labels.
On the other hand we cannot then utilize ``neutral'' examples for measuring annotation \texttt{F1} (since these examples are missing row annotations), but we do not consider this as a drawback of our proposal because \citet{camburu_e-snli_2018} indicated that the annotations for the ``neutral'' examples are curated in a somewhat inconsistent manner from the other example classes, and we thus omit them from the annotation scores.

\paragraph{Hyperparameters:}
Alg. \ref{alg:heuristic-search} takes  a standard OT baseline $\tplan \in \tpoly$ to generate the variate set $\variates(\tplan)$, see \eqref{eqn:heuristic-statistic}. We use a regularized OT (the hyper-parameter-less mirror descent version of standard OT \citep[see][]{scetbon_low-rank_2021} with 20 iterations) to obtain $\tplan$.
The standard OT solution lies on a polytope and has at most $m+n$ non-zeros out of the $mn$ locations \citep{peyre_computational_2020}. The regularized version produces more non-zeros that give more information for  seeking unsaturated locations in \eqref{eqn:heuristic-statistic}.
For the standard OT attention network we do the same with a lower 5 iterations to speed up training.

We obtain $(\tau_1,\tau_2)$ that constrain $\searchtree(k_3,\tau_1,\tau_2)$, as follows.
First compute transport plans $\Pi \in U(a,b)$ by solving~\eqref{eqn:optimal-transport}, using the regularized version see above, and computing $\phi_{ij}^s$ and $\Phi_{ij}$ using \eqref{eqn:heuristic-statistic}. The polytope $\tpoly$ constrains the points $(\phi_{ij}^s, \Phi_{ij})$ to lie in a lower triangular region bounded by the horizontal and vertical axes and a line that runs through $(1,0)$ and $(0,1)$. 
The plan coefficients that are uncertain will lie in a box region bounded by the axes $\tau_1$ and $\tau_2$.
Using SNLI annotations as labels that indicate importance, we  chose $(\tau_1,\tau_2) = (.5,.5)$.
We only consider variates $ij$ in \eqref{eqn:heuristic-statistic} if both $i,j$ do not correspond to stop-words.
The transport plan is marginalized using \texttt{max} across columns/rows to arrive at a vector of coefficients, and
we compute an annotation \texttt{F1} score against the annotations. 
The top $k_2 = $ \texttt{n} plans are scored using the \texttt{BestF1@n} metric, which reports the score of the best plan in the learnt subtree $\esttree(k_1,k_2,k_3,\tau_1,\tau_2)$.

\paragraph{Walkthrough of Alg.~\ref{alg:heuristic-search} using Fig.~\ref{fig:multi-oc-search-tree}:}
At Line 1, the root node (labeled R) and its candidate plan $\hat{\tplan}_1$ are computed, and $\esttree(k_1,k_2,k_3,\tau_1,\tau_2)$ is initialized.
 At Line 2 the single order constraint nodes are constructed and pushed to stack $\mathcal{S}$ along with saturations $\ijdx{\heuristic}{i}{j}$ using ~\eqref{eqn:heuristic-statistic-triple}.
 At Line 4 a variate $\ij{i}{j}$ corresponding to a single order constraint (e.g., ``piece''-``pair'') is popped off, and at Line 6 a new candidate $\hat{\tplan}_2$  is computed.
 At Line 7 this candidate is identified by adding $\ij{i}{j}$ to $\esttree(k_1,k_2,k_3,\tau_1,\tau_2)$ which now has two nodes.
 Then at Lines 11-14, new depth-2 nodes $\nij{i}{j}{2}=\ij{i'}{j'}\ij{i}{j}$ are computed from $\hat{\tplan}_2$ and pushed to $\mathcal{S}$ to be  considered next in $\searchtree(k_3,\tau_1,\tau_2)$ (e.g., ``piece''-``pair'' followed by ``clothing''-``pants'').
  Lines 3-14 iterate to update $\esttree(k_1,k_2,k_3,\tau_1,\tau_2)$ and the candidate set.
 Once $\esttree(k_1,k_2,k_3,\tau_1,\tau_2)$ grows to $k_2$ nodes, the lower bound test at Lines 5 and 11  prune redundant nodes of $\searchtree(k_3,\tau_1,\tau_2)$.
 The iterations terminate when \texttt{count} reaches $k_1$.

\paragraph{Hardware and compute times:} Classifier training was performed on a multi-core Ubuntu virtual machine and on a nVidia Tesla P100-PCIE-16GB GPUs. For the shallow neural nets GPU memory consumption was found $<$ 1GB. Training times for 10 epochs (around 5 million sample runs) were roughly 21 hours. 
Algorithm~\ref{alg:heuristic-search} took about 16 hours to complete over the test set.

\subsection{Color Transfer}

For the source image, we use \citet{felzenszwalb_efficient_2004} to preserve pixel locality \citep[see][]{alvarez-melis_structured_2017}, and for the target we  perform color RGB segmentation via  $k$-means, in the range of 5-10 color clusters. We illustrate the results of Alg.~\ref{alg:heuristic-search} on three examples from learnt subtrees $\esttree(20,5,1,\tau_1,\tau_2)$ and $\esttree(40,10,2,\tau_1,\tau_2)$, where in this set of experiments we chose $(\tau_1,\tau_2) = (0.5,1.0)$. 
Here we have no labels to tune $(\tau_1,\tau_2)$ as we did for e-SNLI. The neighbour saturation metric $\ijdx{\heuristic}{i}{j}$ (see \eqref{eqn:heuristic-statistic-triple}) had less effect in this application. Therefore, it  sufficed here to use non-regularized  OT to compute the base-plan $\tplan \in \tpoly$.
 We use the  5 largest source image segments and the  2 largest target image segments  in terms of  $D_{ij}$; that is, we only consider source image segments  large enough to be visually significant, and   target image segments with contrasting  colors. The WikiArt dataset can be found at \url{https://github.com/cs-chan/ArtGAN/tree/master/WikiArt%20Dataset}.



\subsection{Computational Efficiencies}

\paragraph{Computation Issues} Each computation run of Alg.~\ref{alg:admm} is measured on a single Intel x86 64 bit Xeon 2MHZ with 12GB memory per core. Also, since the python library \texttt{scipy.optimize} depends on \texttt{numpy}\footnote{
    https://numpy.org/
} and can run \texttt{numpy} with multiple threads, for fair comparison with our (single-threaded) implementation of Alg.~\ref{alg:admm}, we disabled all parallelisms.

\paragraph{Effectiveness of Lower Bound~\eqref{eqn:final-bound} in Alg.~\ref{alg:heuristic-search}}
Figure \ref{fig:snli-bounds} demonstrates the effectives of the lower bound~\eqref{eqn:final-bound}, by showing the number of search tree nodes in $\searchtree(k_3, \tau_1, \tau_2)$ that were skipped while performing the e-SNLI experiment for various depths $k_3=1,2,3$ and a large candidate size $k_1 = 40$. The bounds show to be effective in skipping more nodes as the size of $\searchtree(k_3, \tau_1, \tau_2)$ increases with the depth $k_3$.

\begin{figure}
    \centering
    \includegraphics[width=.3\columnwidth]{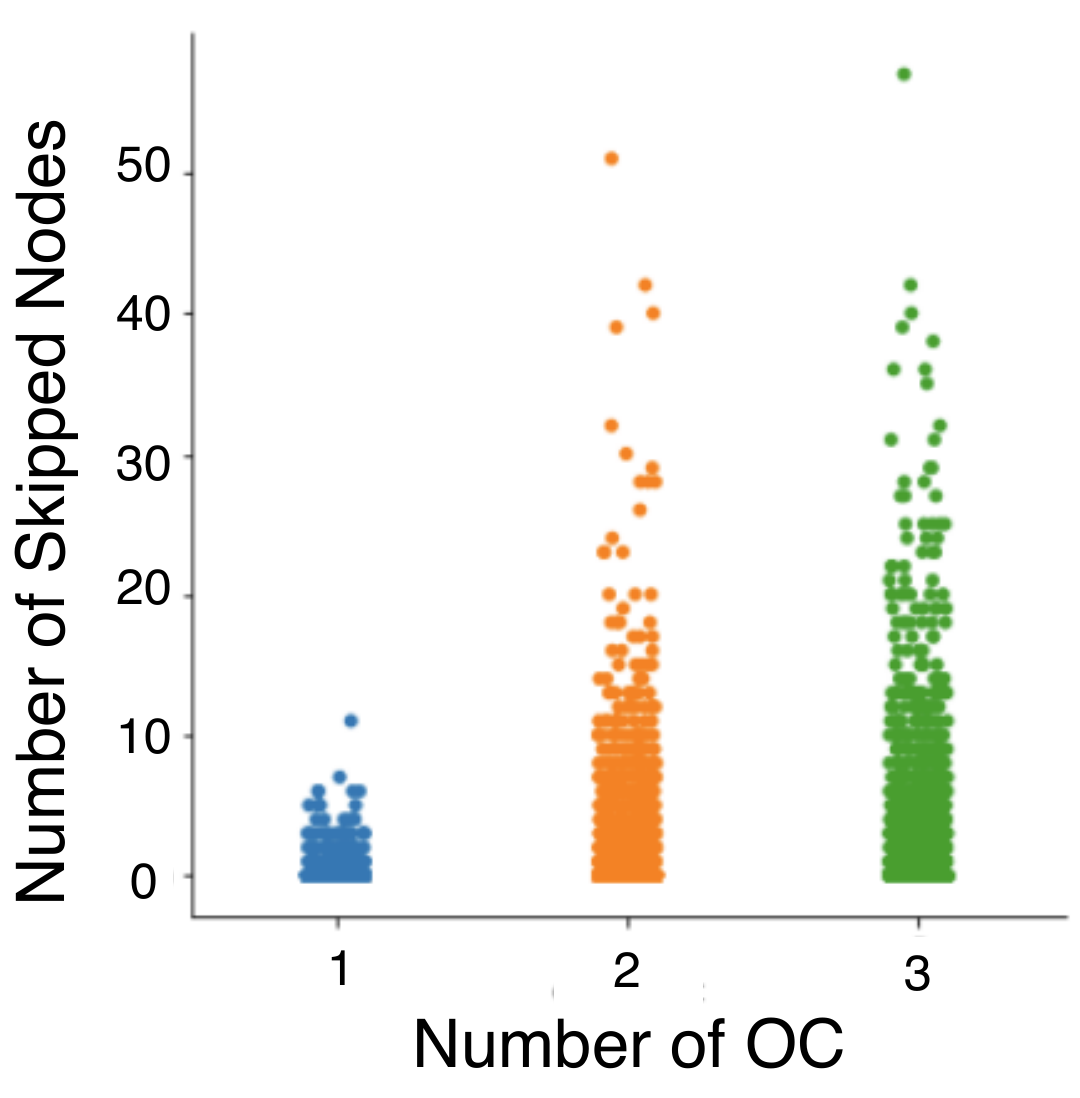}
    \caption{Effectiveness of lower bound~\eqref{eqn:final-bound} in terms of nodes skipped in $\searchtree(k_3, \tau_1, \tau_2)$ for various depths $k_3=1,2,3$.}
    \label{fig:snli-bounds}
\end{figure}

%

\section{Proofs}

First 
 we provide the proofs of the propositions concerning the projections,  Proposition~\ref{prop:euclidean-projector-order}
and  Proposition~\ref{prop:euclidean-projector-sums}. 
Secondly, we cover the bounds,  giving the proof of  correctness of Proposition~\ref{prop:tailfun-bound}. 

\subsection{Projections}


The projection $\proj{\constraintOne \intersect \constraintTwo}{X} = \arg \min_{Y \in \constraintOne} \norm{Y-X}{F}^2$ involves
$mn-k$ constraints $\ijdx{Y}{p}{q} \leq \ijdx{Y}{i_1}{j_1}$ for 
$\ij{p}{q} \in \langidx$ where $\langidx$ in~\eqref{eqn:transport-plan-ordering-constraints},
followed by $mn - k + 1$ constraints $0 \leq \ijdx{Y}{p}{q} $ for $\ij{p}{q} \in \langidx \union \{\ij{i_1}{j_1} \}$,
and $k-1$ constraints $\ijdx{Y}{i_\ell}{j_\ell} \leq \ijdx{Y}{i_{\ell+1}}{j_{\ell+1}}$ for 
$\ell \in \iset{k-1}$.
Write
$\ijdx{\lambda}{p}{q}, \ijdx{\delta}{p}{q}$ and $\eta_\ell$ for their respective Lagrangian variables 
and the Karush-Kuhn-Tucker (KKT) conditions:
\begin{eqnarray}
    \ijdx{Y}{p}{q} 
    &=& \left\{
        \begin{array}{lll}
            \ijdx{X}{p}{q} 
            - 
            \ijdx{\lambda}{p}{q} 
            +
            \ijdx{\delta}{p}{q} 
            & 
            \mbox{for all}
            & 
            \ij{p}{q} \in \langidx
            \\
            \ijdx{X}{i_1}{j_1} 
            - \eta_1
            + 
            \sum_{
                \ij{p}{q} \in \langidx
            }
            \ijdx{\lambda}{p}{q}  
            +
            \ijdx{\delta}{i_1}{j_1} 
            & \mbox{for} 
            &  
            \ij{p}{q}  = \ij{i_1}{j_1}
            \\
            \ijdx{X}{i_\ell}{j_\ell} + \eta_{\ell-1} - \eta_\ell
            & \mbox{for all} 
            &  
            \ij{p}{q}  = \ij{i_\ell}{j_\ell}~~\mbox{where}~~2 \leq \ell \leq k
        \end{array}
    \right. 
    \label{order:eqn:KKT1} 
    \\
    \ijdx{Y}{p}{q} 
    &\leq& 
    \ijdx{Y}{i_1}{j_1} 
    ,~~~\mbox{for all}~~~\ij{p}{q} \in \langidx,~~\mbox{and}~~~~
    \ijdx{Y}{p}{q} \geq 0~~~~~~~~~
    \mbox{for all}~~~\ij{p}{q} \in \langidx \union \{\ij{i_1}{j_1}\}
    \label{order:eqn:KKT2} 
    \\
    \ijdx{\lambda}{p}{q} 
    &\geq& 0
    ,~~~\mbox{for all}~~~\ij{p}{q} \in \langidx,~~~~~~~~\mbox{and}~~~~
    \ijdx{\delta}{p}{q} \geq 0~~~~~~~~~~
    \mbox{for all}~~~\ij{p}{q} \in \langidx \union \{\ij{i_1}{j_1}\}
    \label{order:eqn:KKT3} 
    \\
    \ijdx{\lambda}{p}{q} 
    (
        \ijdx{Y}{p}{q} - 
        \ijdx{Y}{i_1}{j_1} 
    )
    &=& 0
    ,~~~\mbox{for all}~~~\ij{p}{q} \in \langidx,~~~~~~~~\mbox{and}~~~~
    \ijdx{\delta}{p}{q} \ijdx{Y}{p}{q} = 0~~~~
    \mbox{for all}~~~\ij{p}{q} \in \langidx \union \{\ij{i_1}{j_1}\}
    \label{order:eqn:KKT4} 
    \\
    \ijdx{Y}{i_\ell}{j_\ell} 
    &\leq& 
    \ijdx{Y}{i_{\ell+1}}{j_{\ell+1}} 
    ,~~~\mbox{for all}~~~\ell \in \iset{k-1}
    \label{order:eqn:KKT5} 
    \\
    \eta_\ell
    &\geq& 0
    ,~~~\mbox{for all}~~~\ell \in \iset{k-1}
    \label{order:eqn:KKT6} 
    \\
    \eta_\ell
    (
        \ijdx{Y}{i_{\ell+1}}{j_{\ell+1}}
        - 
        \ijdx{Y}{i_\ell}{j_\ell}
    )
    &=& 0
    ,~~~\mbox{for all}~~~\ell \in \iset{k-1}
    \label{order:eqn:KKT7} 
\end{eqnarray}


\begin{lemma}
     \label{order:lem:tau}

     Given any index $\ij{i}{j} = \ij{i_1}{j_1}$ and positive $\eta \geq 0$, any $\ijdx{X}{p}{q} \in \Real$ for all $\ij{p}{q} \in \langidx \union \{\ij{i}{j}\}$,
     let $r = r(\eta) = \rank{\ijdx{X}{i}{j} - \eta}$ over the $mn-k+1$ variates.
     Then for both $\tau(\cdot) = \tau(\cdot, \eta)$ and $0 \leq t = t(\eta) < r$ (see both~\eqref{eqn:order-threshold-tau} and~\eqref{eqn:order-threshold-t} in main text) we have 
     \begin{equation}
          \tau(t,\eta) \geq \odidx{X}{t+1},~~\mbox{and if}~~t>0~~
          \mbox{also}~~
          \tau(t,\eta) \leq \odidx{X}{t} \leq \odidx{X}{t-1} \leq \cdots \leq \odidx{X}{1}.
          \label{order:eqn:condition-A}
     \end{equation}
\end{lemma}
\begin{proof}[Proof of Lemma~\ref{order:lem:tau}]
     Assume $\eta = 0$ and drop $\eta$ from $\tau$ and $t$, since
     the case $\eta > 0$ is equivalent to
     the case
     $\ijdx{X'}{p}{q} = \ijdx{X}{p}{q}$ and $\ijdx{X'}{i}{j} = \ijdx{X}{i}{j} - \eta$.
     Then, the statement~\eqref{order:eqn:condition-A} holds if (i)  $\odidx{X}{t} \geq \tau(t)$ and (ii)  $\tau(t) \geq \odidx{X}{t+1}$.
     Fact (i) is relevant only when $t > 0$, and
     holds by definition (see~\eqref{eqn:order-threshold-t} in main text) because 
     $t+1$ is minimal in $\iset{r}$ for which $ \tau(t+1) > \odidx{X}{t+1}$ holds.
     Fact (ii) holds in either of the possible two cases: 

     \textbf{Case I: [$t+1=r$]}. In this case, note $\tau(t)$ is in fact the unweighted average of $\odidx{X}{t+1} = \ijdx{X}{i}{j}$ and values $\ijdx{X}{p}{q}$ that have higher $\rank{\ijdx{X}{p}{q}} > r$. Therefore we must conclude $\tau(t) \geq \odidx{X}{t+1}$.

     \textbf{Case II: [$t+1 < r$]}. In this case, note that by definition of $t$ we must have $\tau(t+1) > \odidx{X}{t+1}$ is satisfied. 
     Together with $(t+1) \cdot \tau(t+1) = t \cdot \tau(t) + \odidx{X}{t+1}$ we have $\tau(t) > \tau(t+1)$.  
     Therefore we conclude $\tau(t) >\tau(t+1) > \odidx{X}{t+1}$.
\end{proof}


\begin{lemma}
    \label{order:lem:tau2}

    Given any index $\ij{i}{j} = \ij{i_1}{j_1}$ and positive $\eta \geq 0$, any $\ijdx{X}{p}{q} \in \Real$ for all $\ij{p}{q} \in \langidx \union \{\ij{i}{j}\}$, for where
    at least one $X_{pq}$ in $\langidx \union \{\ij{i}{j}\}$ is negative.
    Let $1 \leq s \leq t$ denote the least-ranked element that is negative in the ranking 
    $
        \odidx{X}{1} 
        \geq \cdots \geq 
        \odidx{X}{s-1} 
        \geq 0 
        >
        \odidx{X}{s-1} 
        \geq \cdots \geq 
        \odidx{X}{t} 
    $ for $t = t(\eta)$ satisfying~\eqref{eqn:order-threshold-t} in main text.
    Then for $\tau(\cdot) = \tau(\cdot, \eta)$ in~\eqref{eqn:order-threshold-tau}, if 
    $\tau(t) < 0$, we necessarily have: 
    \begin{equation}
        \ijdx{X}{i}{j} < 0,~~\mbox{and}~~
         \sum_{\ell =1}^{s-1} \odidx{X}{\ell} 
         <
         -\ijdx{X}{i}{j},
         \label{order:eqn:condition-A2}
    \end{equation}
\end{lemma}
\begin{proof}[Proof of Lemma~\ref{order:lem:tau2}]
     Assume $\eta = 0$ and drop $\eta$ from $\tau$ and $t$, since
     the case $\eta > 0$ is equivalent to
     the case
     $\ijdx{X'}{p}{q} = \ijdx{X}{p}{q}$ and $\ijdx{X'}{i}{j} = \ijdx{X}{i}{j} - \eta$.
     We get that 
     $
        \tau(s) = 
        (s+1)^{-1} \cdot 
        (
            \sum_{\ell =1}^{s-1} \odidx{X}{\ell} + 
            \odidx{X}{s}. + 
            \ijdx{X}{i}{j} 
        )
        \leq
        \odidx{X}{s}
     $, 
     where the inequality follows from
    the minimality of $t$ in \eqref{eqn:order-threshold-t}.
    Re-arranging, we get $
        \sum_{\ell =1}^{s-1} \odidx{X}{\ell} + 
        \ijdx{X}{i}{j} 
        \leq
        \odidx{X}{s}
    $. 
    On the other hand, we had assumed $\odidx{X}{s} < 0$, 
    so therefore we must conclude \eqref{order:eqn:condition-A2}.
\end{proof}

\begin{lemma}
    \label{order:lem:first}
    Given index $\ij{i}{j} = \ij{i_1}{j_1}$ and
    $\eta \geq 0$,
    coefficients $\ijdx{X}{p}{q} \in \Real$ for all $\ij{p}{q} \in \langidx \union \{\ij{i}{j}\}$,
    let $r = r(\eta) = \rank{\ijdx{X}{i}{j} - \eta}$ 
    over the $mn-k+1$ coefficients.
    Let $\tau(\cdot) = \tau(\cdot, \eta)$ and $0 \leq t = t(\eta)$ respectively satisfy~\eqref{eqn:order-threshold-tau} and~\eqref{eqn:order-threshold-t} in the main text.
    For all $\ij{p}{q} \in \langidx \union \{\ij{i}{j}\}$,
    set $\ijdx{Y}{p}{q}$ as: i) if $\rank{\ijdx{X}{p}{q}} \leq t$ or $\ij{p}{q} = \ij{i}{j}$ then set $\ijdx{Y}{p}{q} = \zeroprog{\tau(t,\eta)}$, and ii) if otherwise then set $\ijdx{Y}{p}{q} = \zeroprog{\ijdx{X}{p}{q}} $.
    Then the $mn-k$ and $mn -k + 1$ coefficients $\ijdx{\lambda}{p}{q}$ and $\ijdx{\delta}{p}{q}$ as determined 
    by the first two lines of~\eqref{order:eqn:KKT1}, given this choice for $\ijdx{Y}{p}{q}$ and $\ijdx{X}{p}{q}$, satisfy~\eqref{order:eqn:KKT2}--\eqref{order:eqn:KKT4} together with the chosen $\ijdx{Y}{p}{q}$ and $\ijdx{X}{p}{q}$.
\end{lemma}

\begin{proof}[Proof of Lemma~\ref{order:lem:first}]
Fix any $\ij{i}{j} = \ij{i_1}{j_1}$.
Assume $\eta = 0$, because
the general case $\eta \geq 0$ is equivalent to
the case
$\ijdx{X'}{p}{q} = \ijdx{X}{p}{q}$ and $\ijdx{X'}{i}{j} = \ijdx{X}{i}{j} - \eta$, hence we simply write $\tau(\cdot) = \tau(\cdot,\eta)$ and $t = t(\eta)$.
We show that~\eqref{order:eqn:condition-C1}-\eqref{order:eqn:condition-C2} below
\begin{eqnarray}
     (\mbox{for all}~\ij{p}{q} \in \langidx)~~~~~~~~~~~~~~~~~
     \ijdx{Y}{p}{q}
     &=& \left\{ 
          \begin{array}{cc}
          \zeroprog{\tau(t)} &  \mbox{if}~\rank{\ijdx{X}{p}{q}} \leq t
          \\
          \zeroprog{\ijdx{X}{p}{q}} & \mbox{otherwise}
          \end{array}
     \right.
     \label{order:eqn:condition-C1}
     \\
     \ijdx{Y}{i}{j} 
     &=& \zeroprog{\tau(t)},
     \label{order:eqn:condition-C2}
\end{eqnarray}
together with~\eqref{order:eqn:condition-A} and~\eqref{order:eqn:condition-A2} in Lemmatta~\ref{order:lem:tau} and~\ref{order:lem:tau2} respectively,
establish $\ijdx{\lambda}{p}{q}$
and $\ijdx{\delta}{p}{q}$ that relate $\ijdx{Y}{p}{q}$ and $\ijdx{X}{p}{q}$ by (the first two equations of)~\eqref{order:eqn:KKT1}, and satisfy
the three KKT conditions~\eqref{order:eqn:KKT2},~\eqref{order:eqn:KKT3} and~\eqref{order:eqn:KKT4} for $ i_i j_1 = ij$, proving the lemma.

\noindent\textbf{KKT condition~\eqref{order:eqn:KKT2}}. 
We first show $\ijdx{Y}{p}{q} \leq  \ijdx{Y}{i}{j} $
for all $\ij{p}{q} \in \langidx$. Consider $\rank{\ijdx{X}{p}{q}} > t$.
Then wherever $\ijdx{X}{p}{q} \geq 0$ we have
$
     0 \leq \ijdx{Y}{p}{q} 
     \stackrel{
          \eqref{order:eqn:condition-C1}
     }{=}
     \ijdx{X}{p}{q} 
     \leq
     \odidx{X}{t+1}
     \stackrel{
          \eqref{order:eqn:condition-A}
     }{\leq}
     \tau(t) 
     \stackrel{
          \eqref{order:eqn:condition-C2}
     }{=}
     \ijdx{Y}{i}{j} 
$
, and whenever $\ijdx{X}{p}{q} < 0$ we have
$
     \ijdx{Y}{p}{q} 
     \stackrel{
          \eqref{order:eqn:condition-C1}
     }{=}
     \zeroprog{\ijdx{X}{p}{q}}
     = 0
     \leq
     \zeroprog{\tau(t)}
     = 
     \ijdx{Y}{i}{j} 
$.
For all other 
$\ij{p}{q} \in \langidx$
where $\rank{\ijdx{X}{p}{q}} \leq t$ we see that
$
     \ijdx{Y}{p}{q} 
     \stackrel{
          \eqref{order:eqn:condition-C1}
     }{=}
     \zeroprog{\tau(t)}
     \stackrel{
          \eqref{order:eqn:condition-C2}
     }{=}
     \ijdx{Y}{i}{j} 
$
.
Thus the first set of inequalities in \eqref{order:eqn:KKT2} follow.
The second set $\ijdx{Y}{p}{q} \geq 0$ for all $\langidx \union \{\ij{i}{j}\}$ follow 
by definitions~\eqref{order:eqn:condition-C1} and~\eqref{order:eqn:condition-C2}.

We next show~\eqref{order:eqn:KKT3} and~\eqref{order:eqn:KKT4} seperately for $\ij{p}{q} \in \langidx$; 
for $\ij{p}{q} \in \langidx$ we choose $\ijdx{\lambda}{p}{q},\ijdx{\delta}{p}{q}$ satisfying 
$
    \ijdx{X}{p}{q}
    + 
    \ijdx{\delta}{p}{q}
    = 
    \ijdx{\lambda}{p}{q}
    + 
    \ijdx{Y}{p}{q}
$
as
\begin{eqnarray}
    (\mbox{if}~~\rank{\ijdx{X}{p}{q}} > t)
    &\Rightarrow&
        \ijdx{\lambda}{p}{q} = 0,~~\mbox{and}~~
        \ijdx{\delta}{p}{q}
        = 
        \zeroprog{\ijdx{X}{p}{q}}
        - 
        \ijdx{X}{p}{q},
        \label{order:eqn:condition-E}
        \\
    (\mbox{if}~~\rank{\ijdx{X}{p}{q}} \leq t)
    &\Rightarrow&
    \left\{ 
         \begin{array}{lll}
         \ijdx{\lambda}{p}{q} = 
         \ijdx{X}{p}{q} - 
         \zeroprog{\tau(t)},
         &
         \ijdx{\delta}{p}{q} = 0,
         &
         \mbox{if}~\tau(t) \geq 0 ~\mbox{or}~ 
         \ijdx{X}{p}{q} \geq 0,
         \\
         \ijdx{\lambda}{p}{q} = 0,
         &
         \ijdx{\delta}{p}{q} = 
         - \ijdx{X}{p}{q},
         &
         \mbox{if}~\tau(t) < 0 ~\mbox{and}~ 
         \ijdx{X}{p}{q} < 0
         \end{array}
    \right.
    \label{order:eqn:condition-D}
\end{eqnarray}
Verify \eqref{order:eqn:condition-E} satisfies
(first equation of)~\eqref{order:eqn:KKT1} by putting
$\ijdx{Y}{p}{q} = \zeroprog{\ijdx{X}{p}{q}}$ as in~\eqref{order:eqn:condition-C1}, 
and similarly verify 
\eqref{order:eqn:condition-D} satisfies the same by putting
$\ijdx{Y}{p}{q} = \zeroprog{\tau(t)}$ whilst considering both cases $\tau(t) \geq 0$ and  $\tau(t) < 0$.
Following that~\eqref{order:eqn:KKT3} and~\eqref{order:eqn:KKT4} for the remaining index
$\ij{i}{j}$ will be shown after.

\noindent \textbf{KKT condition~\eqref{order:eqn:KKT3} for $\ij{p}{q} \in \langidx$}. 
Consider $\rank{\ijdx{X}{p}{q}} > t$.
From~\eqref{order:eqn:condition-E}, this is trivally satisfied for $\ijdx{\lambda}{p}{q}$, and holds for 
$
\ijdx{\delta}{p}{q}
$
because 
$
    \zeroprog{\ijdx{X}{p}{q}}
    - 
    \ijdx{X}{p}{q}
    \geq 
    0
$.
For $\rank{\ijdx{X}{p}{q}} \leq t$
we see from~\eqref{order:eqn:condition-D} 
the following.
We see $ \ijdx{\lambda}{p}{q} = \ijdx{X}{p}{q} - \zeroprog{\tau(t)} \geq 0$ holds in the event where
\emph{either} $\tau(t) \geq 0 $
or $\ijdx{X}{p}{q} \geq 0$ holds, 
since 
$
\tau(t)
\stackrel{\eqref{order:eqn:condition-A}}{\leq}
\ijdx{X}{p}{q} 
$
whenever
$\rank{\ijdx{X}{p}{q}} \leq t$.
Also in the event where \emph{both} $\tau(t) <0 $ and $X_{pq} <0$ hold, we have that $ \delta_{pq}  = - \ijdx{X}{p}{q} \geq 0 $ because then we have $\ijdx{X}{p}{q}$ to be negative.

\noindent\textbf{KKT condition~\eqref{order:eqn:KKT4} for $\ij{p}{q} \in \langidx$}.
The first set of equalities in~\eqref{order:eqn:KKT4} satisfy as follows.
Whenever 
$\rank{\ijdx{X}{p}{q}} \leq t$, we have
$
    \ijdx{Y}{p}{q}
    =
    \ijdx{Y}{i}{j},
$
see above discussion on KKT condition~\eqref{order:eqn:KKT2}. 
For $\rank{\ijdx{X}{p}{q}} > t$, we had chosen
$
    \ijdx{\lambda}{p}{q} = 0
$,
see~\eqref{order:eqn:condition-E}.
The second set of equalities in~\eqref{order:eqn:KKT4} satisfy as follows: for case 
$\rank{\ijdx{X}{p}{q}} > t$ this is seen given choice \eqref{order:eqn:condition-C1}
putting
$
    \ijdx{Y}{p}{q} =  
    \zeroprog{\ijdx{X}{p}{q}} 
$ 
and choice
\eqref{order:eqn:condition-E}
for
$\ijdx{\delta}{p}{q}$,
and for 
$\rank{\ijdx{X}{p}{q}} \leq  t$ this is seen given
\eqref{order:eqn:condition-C1}
putting
$
    \ijdx{Y}{p}{q} =  
    \zeroprog{\tau(t)} 
$ 
and choice
\eqref{order:eqn:condition-D}
for $\ijdx{\delta}{p}{q}$.

It remains to show KKT conditions
$
    \ijdx{\delta}{i}{j} \geq 0
$
in
\eqref{order:eqn:KKT3} and
$
    \ijdx{\delta}{i}{j} 
    \ijdx{Y}{i}{j} = 0
$
in~\eqref{order:eqn:KKT4}.
Derive the following:
\[
     \sum_{
        \ij{p}{q} \in \langidx
    }
    \ijdx{\lambda}{p}{q}
    \stackrel{\eqref{order:eqn:condition-E}}{=}
    \sum_{
        \ij{p}{q} \in \langidx:~
        \rank{\ijdx{X}{p}{q}}  \leq t
    } 
    (
        \ijdx{X}{p}{q}
        +
        \ijdx{\delta}{p}{q}
        - 
        \zeroprog{
            \tau(t)
        }
    ) 
    ~~~~=~~~~
    \sum_{
        \ij{p}{q} \in \langidx:~
        \rank{\ijdx{X}{p}{q}}  \leq t
    } 
    \left(
        \ijdx{X}{p}{q}
        +
        \ijdx{\delta}{p}{q}
    \right) 
    - 
    t 
    \cdot 
    \zeroprog{
        \tau(t)
    }.
\]
Then put
$
    \ijdx{Y}{i}{j}
    =
    \zeroprog{
        \tau(t)
    }
$
from~\eqref{order:eqn:condition-C2} 
in
$
    \ijdx{\delta}{i}{j}
    = 
    \ijdx{Y}{i}{j}
    -
    \ijdx{X}{i}{j}
    - \sum_{\ij{p}{q \in \langidx}} 
    \ijdx{\lambda}{i}{j}
$,
and put in the derivation above 
to get
$
    \ijdx{\delta}{i}{j} = 
    (t+1) \cdot 
    \zeroprog{
        \tau(t)
    }
    - 
    (
        \ijdx{X}{i}{j}
        +
        \sum_{
            \ij{p}{q} \in \langidx:~
            \rank{\ijdx{X}{p}{q}}  \leq t
        } 
        (
            \ijdx{X}{p}{q}
            +
            \ijdx{\delta}{p}{q}
        )
    )
$, and simplify to:
\begin{eqnarray}
    \ijdx{\delta}{i}{j}
    &=& \left\{ 
         \begin{array}{ll}
         (t+1)\tau(t) - 
         (
            \ijdx{X}{i}{j} 
            +
            \sum_{
                \ij{p}{q} \in \langidx:~
                \rank{\ijdx{X}{p}{q}}  \leq t
            } 
            \ijdx{X}{p}{q} 
         )
         \stackrel{\eqref{eqn:order-threshold-tau}}{=}
         (t+1)\tau(t) - (t+1) \tau(t) 
         =
         0
         &
         \mbox{if}~\tau(t) \geq 0
         \\
         -\ijdx{X}{i}{j} 
         -
         \sum_{
             \ij{p}{q} \in \langidx:~
             \rank{\ijdx{X}{p}{q}}  \leq t
         } 
         (
            \ijdx{X}{p}{q} 
            +
            \ijdx{\delta}{p}{q} 
         )
         \stackrel{\eqref{order:eqn:condition-D}}{=}
         -\ijdx{X}{i}{j} 
         -
         \sum_{
            \ij{p}{q} \in \langidx:~
            \ijdx{X}{p}{q} \geq 0
         } 
            \ijdx{X}{p}{q} 
        \stackrel{\eqref{order:eqn:condition-A2}}{>}
         0,
         &
         \mbox{if}~\tau(t) < 0 
         \end{array}
    \right.
    \nonumber 
\end{eqnarray}
showing both \eqref{order:eqn:KKT3} and \eqref{order:eqn:KKT4} hold, where we had used 
\eqref{order:eqn:condition-A2} from Lemma~\ref{order:lem:tau2} for the final inequality in the 
$\tau(t) < 0$ case.
\end{proof}


\begin{lemma}
    \label{order:lem:tau-convex}
    For $\tau(\cdot, \eta)$ and $t(\eta)$ in~\eqref{eqn:order-threshold-tau} and~\eqref{eqn:order-threshold-t} in the main text, consider the function
    $T: \Real_+ \mapsto \Real$ defined as $T(\eta) := \zeroprog{\tau(t(\eta),\eta)}$ for any $\eta \geq 0$.
    Then $T$ is piecewise-linear, monotonic non-increasing and convex in $\eta \geq 0$.
\end{lemma}


\begin{proof}[Proof of Lemma~\ref{order:lem:tau-convex}]
     $\tau(s,\eta)$ is decreasing linear function of $\eta$ for fixed $s$, therefore $T(\eta) = \zeroprog{\tau(t(\eta),\eta)} = \max(\tau(t(\eta),\eta) ,0)$ is a convex, piecewise linear function with inflection points whenever $t(\eta)$ changes value, 
     with the final inflection point occurring when $\tau(t(\eta),\eta)$ meets the horizontal axis.
\end{proof}


\begin{lemma}[\citep{grotzinger_projections_1984}]
    \label{order:lem:pava}
    For $p,q,r,s \in \iset{k}$ satisfying $ 1 \leq p \leq q \leq r \leq s \leq k$ where $r=q+1$, assume 
    $ \ijdx{\pmean}{p}{q} \leq \ijdx{\pmean}{s}{q}$ for $\pmean$ in~\eqref{order:eqn:pava}.
    Let $\ell$ satisfy $p \leq \ell \leq s$.
    Then for all $\ell \leq q$ we have $\ijdx{\psum}{p}{\ell}  - (\ell-p + 1) \ijdx{\pmean}{p}{s}  \geq 0 $, and for all $\ell > q$ we have
    we have $\ijdx{\psum}{r}{\ell}  - (\ell-r + 1) \ijdx{\pmean}{r}{s}  \geq 0 $.
\end{lemma}



\paragraph{Restatement of  Proposition~\ref{prop:euclidean-projector-order}}


    For $\constraintOne = \nijdx{\order}{i}{j}{k}$ for any 
    $\ij{i_\ell}{j_\ell} \in \ijset{m}{n}$ where $ \ell \in \iset{k}$,
    consider the Euclidean projector $\proj{\constraintOne}{X}$ for any $X \in \mnReal$.
    Let $T(\eta) := \zeroprog{\tau(t(\eta),\eta)}$ for 
    $\tau,t$ in~\eqref{eqn:order-threshold-tau} and~\eqref{eqn:order-threshold-t} in the main text, and let $\langidx$ in~\eqref{eqn:transport-plan-ordering-constraints}.
    Then for any $X \in \mnReal$, ePAVA will successfully terminate with some $\bl, \tilde{\eta}$, \textup{\texttt{le}},\textup{\texttt{ri}}, and \textup{\texttt{val}}. 
    Furthermore, the projection $\hat{X} = \proj{\constraintOne}{x}$ 
    satisfies i) for $\ij{p}{q} \in \langidx$ we have $\ijdx{\hat{X}}{p}{q} = T(\tilde{\eta}) = \textup{\texttt{val}}[1]$ if $\rank{\ijdx{\hat{X}}{p}{q}} \leq t(\tilde{\eta})$ or $\ijdx{\hat{X}}{p}{q} = \ijdx{X}{p}{q}$ otherwise, and 
    ii) for  $\ell \in \iset{k}$
    we have $\ijdx{\hat{X}}{i_\ell}{j_\ell} = \textup{\texttt{val}}[\bl']$ iff $\textup{\texttt{le}}[\bl'] \leq \ell \leq \textup{\texttt{ri}}[\bl']$ for some $\bl' \in \iset{\bl}$.


\begin{proof}

Proving Proposition~\ref{prop:euclidean-projector-order} involves exhibiting a $Y \in \mnReal$ satisfying~\eqref{order:eqn:KKT1}--\eqref{order:eqn:KKT7}. 
Given variates $\ijdx{X}{i_\ell}{j_\ell}$ where $\ij{i_\ell}{j_\ell} \notin \langidx$ for all $\ell \in \iset{k}$, define:
\begin{equation}
    \ijdx{\psum}{p}{q} = \sum_{\ell=p}^q \ijdx{X}{i_\ell}{j_\ell},~~~
    \ijdx{\pmean}{p}{q} = \frac{\ijdx{\psum}{p}{q}}{q - p + 1},~~~
    \mbox{for}~1 \leq p \leq q \leq k.
    \label{order:eqn:pava}
\end{equation}


The block of operations between Lines 4 and 10 of Alg.~\ref{alg:pava} is termed ``coalescing''~\citep{grotzinger_projections_1984}. 
As coalescing occurs during iterates, in the case whenever $\bl=2$, notice the parameter denoted $\tilde{\eta}$ in Alg.~\ref{alg:pava}  (Line 7) that updates.
We prove that if this parameter is taken to be $\eta_1$ in the KKT conditions~\eqref{order:eqn:KKT1} and~\eqref{order:eqn:KKT6}--\eqref{order:eqn:KKT7}, 
that Alg.~\ref{alg:pava} terminates and when it does with $\hat{X}$,
that $\hat{X}=Y$ satifies the KKT.
The proof is recursive in $\eta_1$ and we initialize $\eta_1 = 0$. 
We use $\eta_\ell$ and $\eta_\ell'$ to denote the current, and next iterate, respectively.

We first invoke~\eqref{order:eqn:KKT1} to fix the relationship between solution and Lagrangian multipliers.
For a given $\eta_1$ value, put $\eta_1 = \eta$ in Lemma~\ref{order:lem:first} to get $mn - k + 1$ coefficients $\ijdx{Y}{p}{q}$ (and the Lagrangians $\ijdx{\lambda}{p}{q},\ijdx{\delta}{p}{q}$) for all $\ij{p}{q} \in \langidx \union \{\ij{i_1}{j_1}\}$ 
satisfying the first two lines of~\eqref{order:eqn:KKT1} and~\eqref{order:eqn:KKT2}--\eqref{order:eqn:KKT4} of the KKT conditions.
Likewise for $\eta_1$, we use~\eqref{order:eqn:KKT1} to derive the $k-1$ Lagrangians that accompany $\ijdx{Y}{i_\ell}{j_\ell}$ as follows.
For any $\bl' \in \iset{\bl}$, ePAVA in Alg.~\ref{alg:pava} sets $\ijdx{Y}{i_\ell}{j_\ell} = \ijdx{Y}{i_{\ell+1}}{j_{\ell+1}}$ for values of $\ell$ that satisfy $\texttt{le}[\bl'] \leq \ell < \ell + 1 \leq \texttt{ri}[\bl']$.
On the other hand if $\ell$ is on the boundary $\ell = \texttt{ri}[\bl']$ and $\ell+1 = \texttt{le}[\bl'+1]$ then ePAVA results in $\ijdx{Y}{i_\ell}{j_\ell} \neq \ijdx{Y}{i_{\ell+1}}{j_{\ell+1}}$.
Therefore using $\ijdx{\psum}{p}{q}$ and $\ijdx{\pmean}{p}{q}$ from~\eqref{order:eqn:pava} and $\nu = \sum_{\ij{p}{q} \in \langidx} \ijdx{\lambda}{p}{q} + \ijdx{\delta}{i_1}{j_1}$, we express 
$\ijdx{Y}{i_\ell}{j_\ell}$ and $\eta_\ell$ for all $\ell \in \iset{k}$:
\begin{eqnarray}
    \eta_\ell &=& \ijdx{\psum}{1}{\ell} - \ell \cdot T(\eta_1) + \nu,
    ~~~~~
    \ijdx{Y}{i_\ell}{j_\ell} = T(\eta_1),~~~~~~\mbox{if}~~
    1 \leq \ell \leq \texttt{ri}[1],
    \label{order:eqn:eta1}
    \\
    \eta_\ell &=& \ijdx{\psum}{p}{\ell} - (\ell-p + 1) \ijdx{\pmean}{p}{q},
    ~~~
    \ijdx{Y}{i_\ell}{j_\ell} = \ijdx{\pmean}{p}{q},~~~~~~\mbox{if}~~
    p = \texttt{le}[\bl'] \leq \ell \leq q = \texttt{ri}[\bl']
    \label{order:eqn:eta2}
\end{eqnarray}
where~\eqref{order:eqn:eta2} is satisfied for all  $ 1 < \bl' \leq \bl$, and simply let $\eta_k = 0$ since this does not contradict~\eqref{order:eqn:pava}.
Thus it remains to show, as the iterates of $\eta_1$ update the values of $\ijdx{Y}{p}{q},\ijdx{\lambda}{p}{q},\ijdx{\delta}{p}{q},\ijdx{\eta}{p}{q}$ by Lemma~\ref{order:lem:first} and~\eqref{order:eqn:eta1}--\eqref{order:eqn:eta2},
the remaining KKT conditions ~\eqref{order:eqn:KKT5}--\eqref{order:eqn:KKT7} are satisfied upon termination of ePAVA.
To this end we  must  prove the existence and required properties of the zero (\ie,~$\tilde{\eta}$, taken here to be $\eta_1$) in Line 7.
Specifically, 
let $q = \texttt{ri}[1]$ be the boundary of block 1.
Suppose either a) $\eta_1=0$ and $q=1$, or b) $\ijdx{\psum}{2}{q} - (q-1) \cdot T(\eta_1) + \eta_1  = 0$ is satisfied for some $\eta_1 > 0$ and $q > 1$.
Let $r=q+1$ and 
consider $\ijdx{\psum}{2}{s}$ from~\eqref{order:eqn:pava} for some $s \geq r \geq 2$.
We show below that in either case a) or b), that if $\ijdx{\pmean}{r}{s} \leq T(\eta_1)$ holds, we must have that i)
there exists a zero $\eta'_1$ satisfying $\ijdx{\psum}{2}{s} - (s-1) \cdot T(\eta'_1) + \eta'_1  = 0$, and ii) $T(\eta'_1) \leq T(\eta_1)$ and $\eta'_1 \geq \eta_1$.

The proof of the above properties of the zero follow. 
In the case $\eta_1 = 0$ and $q=1$, then i) follows by equivalently showing if $T(\eta'_1) - \eta'_1 / (s-1) - \ijdx{\pmean}{2}{s}$ has a zero (divide by $(s-1) \geq 1$ and use \eqref{order:eqn:pava}). 
Indeed, the zero exists at some $\eta_1' \geq 0$ since the non-increasing $T(\eta'_1)$, see Lemma~\ref{order:lem:tau-convex}, and an increasing linear function $\eta'_1/(s-1) + \ijdx{\pmean}{2}{s}$, meets, as the latter starts at a point lower than the former as given by the assumption $\ijdx{\pmean}{2}{s} \leq T(0)$. Furthermore, ii) holds since we showed $\eta_1' \geq 0 = \eta_1$ and by monotonicity of $T(\cdot)$.
In the other case 
put
$
    \ijdx{\pmean}{2}{s} = 
    (1-\beta)\ijdx{\pmean}{2}{q}
    + 
    \beta \ijdx{\pmean}{r}{s}
$
for $ 0 \leq \beta = (s-q) / (s-1) \leq 1$, and use the condition b) above to derive
$
    T(\eta_1)
    - \ijdx{\pmean}{2}{s} - \eta_1/(s-1)
    = 
    \beta
    \cdot
    (
        T(\eta_1)
        - 
        \ijdx{\pmean}{r}{s} 
    )
    \geq 0
$
where the inequality follows since we assumed $\ijdx{\pmean}{r}{s} \leq T(\eta_1)$.
Then by monotonicity of $T(\cdot)$, see Lemma~\ref{order:lem:tau-convex}, the zero exists and occurs at some point $\eta_1' \geq \eta_1$ with similar arguments as before showing i) and ii).

\noindent\textbf{KKT condition~\eqref{order:eqn:KKT6} for block 1, case~\eqref{order:eqn:eta1}}:
We now show that for block 1, coalescing recursively maintains non-negativity~\eqref{order:eqn:KKT6}.
Suppose $1,q$ and $r,s$ are boundaries of blocks $1$ and $2$, where $\texttt{ri}[1]=q=r-1$.
Let $\eta'_\ell$ for $\ell \in \iset{s}$ equal~\eqref{order:eqn:eta1} after coalescing, rewritten into $\eta'_\ell = \ijdx{\psum}{2}{\ell} - (\ell-1) T(\eta'_1) + \eta'_1 $, and similarly $\eta_\ell$ for $\ell \in \iset{q}$ denotes~\eqref{order:eqn:eta1} before coalescing.
By recursion assumption either $\eta_1 = 0 $ or
$\ijdx{\psum}{2}{q} - (q-1) \cdot T(\eta_1) + \eta_1  = 0$ is satisfied for some $\eta_1 > 0$.
ePAVA coalesces blocks 1 and 2 if $\texttt{val}[1] \leq \texttt{val}[2]$, or equivalently, if $ \ijdx{\pmean}{r}{s} \leq T(\eta_1)$ for value $\eta_1 \geq 0$ , 
and therefore by recursion assumption, 
we conclude a new zero $\eta'_1 \geq 0$ of Line 7 with properties i) and ii) above exists.
Then for $\ell \leq q$ we have  $\eta'_\ell - \eta_\ell = (\ell-1)(T(\eta_1) - T(\eta_1')) + \eta'_1 - \eta_1$, and by the inequality $\eta'_1 - \eta_1 \geq 0$, we conclude $\eta'_\ell - \eta_\ell \geq 0$.
Next for $\ell > q$, we express the following for some $\alpha = 1 - (\ell-1) / (s-1) \geq 0$:
\begin{eqnarray}
    \eta'_\ell - \eta_\ell 
    &\stackrel{(a)}{=}&
    \eta'_1 + \ijdx{\psum}{2}{q} - (\ell -1) \cdot T(\eta_1')
    + (\ell - q) \ijdx{\pmean}{r}{s} 
    \stackrel{(b)}{=}
    \alpha
    \cdot
    \left(
        \eta'_1 +  \ijdx{\psum}{2}{q} - (q-1) \ijdx{\pmean}{r}{s}
    \right) 
    \nonumber \\
    &\stackrel{(c)}{\geq}& 
    \alpha
    (q-1)
    \cdot 
    (T(\eta_1)- \ijdx{\pmean}{r}{s}) 
    \geq 0, \nonumber
\end{eqnarray}
where (a) follows from $\eta'_\ell$ in~\eqref{order:eqn:eta1} and $\eta_\ell$ in~\eqref{order:eqn:eta2},
and (b) follows from expressing $
T(\eta'_1) = 
[\ijdx{\psum}{2}{q} 
+ (s- q) \cdot \ijdx{\pmean}{r}{s}
+ \eta'_1
] / (s-1)
$ and collecting terms into $\alpha$,
and (c) follows  as
$
\eta'_1 + \ijdx{\psum}{2}{q}
\geq 
\eta_1 + \ijdx{\psum}{2}{q} = (q-1) \cdot T(\eta_1),
$
and finally the last inequality follows because we only coalesce when
$ \ijdx{\pmean}{r}{s} \leq T(\eta_1)$.

\noindent\textbf{KKT condition~\eqref{order:eqn:KKT6} for block $> 1$, case~\eqref{order:eqn:eta2}}: We show the similar non-negativity property for other blocks.
Suppose $p,q$ and $r,s$ are boundaries of two blocks $\bl'$ and $\bl'+1$, where $r=q+1$ and $\bl' > 1$.
ePAVA coalesces block $\bl'$ and $\bl'+1$
only
if $\texttt{val}[\bl'-1] \leq \texttt{val}[\bl']$, or equivalently, $ \ijdx{\pmean}{p}{q} \leq \ijdx{\pmean}{r}{s}$.
Consider any $p \leq \ell \leq s$. Eqn.~\eqref{order:eqn:eta2} implies that the Lagrangians before and after coalescing are $\eta_\ell = \ijdx{\psum}{p}{\ell}  - (\ell-p + 1) \ijdx{\pmean}{p}{s}$ and $\eta'_\ell = \ijdx{\psum}{r}{\ell}  - (\ell-r + 1) \ijdx{\pmean}{r}{s}$, respectively.
Supposing $\eta_\ell \geq 0$ and
by these expressions for $\eta'_\ell,\eta_\ell$, we thus invoke Lemma~\ref{order:lem:pava} to show $\eta'_\ell \geq 0$ for all $\ell \leq q$ hold after coalescing;
the other case $\ell > q$ also holds similarly by  Lemma~\ref{order:lem:pava}.

\noindent \textbf{KKT condition~\eqref{order:eqn:KKT7}}: We show the coalescing update Lines 4-10, maintains the boundary $s = \texttt{ri}[\bl']$ for any $\bl' \in \iset{\bl}$ property $\eta_s=0$; this proves complementary slackness~\eqref{order:eqn:KKT7} because~\eqref{order:eqn:eta1}--\eqref{order:eqn:eta2} shows $\ijdx{Y}{i_\ell}{j_\ell}$ to differ only across boundaries.
Let $\eta_q$ and $\eta'_s$ denote boundary Lagrangians for the previous and current coalescing update, respectively; the boundaries $p,q$ and $r,s$ for $r=q+1$ are coalesced.
For~\eqref{order:eqn:eta2} for $\bl' > 1$, we have for $\eta'_s = \ijdx{\psum}{p}{s} - (s - p + 1) \ijdx{\pmean}{p}{s}$ at the boundary $\ell=s$, and $\eta_s =0$ by definition~\eqref{order:eqn:pava}.
For~\eqref{order:eqn:eta1} for $\bl'=1$, rewrite~\eqref{order:eqn:eta1} to get the form $ \eta'_s = \ijdx{\psum}{2}{s} - (s-1) \cdot T(\eta'_1) + \eta'_1 $ resembling the equation with the zero shown above. 
If the coalescing  update is executed for the first time, then $\eta_1 = 0$ and $q=1$.
Otherwise $\eta_1 > 0$ and $q > 1$ and by recursion assumption $ \eta_q = \ijdx{\psum}{2}{q} - (q-1) \cdot T(\eta_1) + \eta_1 = 0$.
Therefore in either cases, together with the condition $ \ijdx{\pmean}{r}{s} \leq T(\eta_1)$ that holds for a coalescing step to occur, we conclude by zero property i) that $ \eta'_s =0$.

\noindent\textbf{KKT condition~\eqref{order:eqn:KKT5}}:
Each coalescing step of ePAVA attempts to restore a non-increasing property of $\ijdx{Y}{i_\ell}{j_\ell}$ for all $\ell \in \iset{k}$.
This is only possible if the new values $\ijdx{Y}{i_\ell}{j_\ell}$  that result from coalescing does not increase beyond the values before coalescing.
This is obvious for the case of $\bl > 2$ by definition~\eqref{order:eqn:pava}.
ePAVA eventually 
terminates satisfying~\eqref{order:eqn:KKT5}, since there are finite number of coalescsing steps, and in the worst case arrives at the terminating state $\bl=1$ and $\texttt{ri}[1]=k$ is arrived at, which satisfies~\eqref{order:eqn:KKT5}.
\end{proof}


We now turn to  the Euclidean projection $\proj{\constraintThree}{X}$ of a matrix $X\in\mnReal$ onto 
the row- and column-sum constraint set $\constraintThree$ for measures $a,b$.
Let $\kron$ denote the matrix Kronecker (\ie, tensor) product.
Observe that the row-sums of any $X \in \mnReal$ can also be obtained via Kronecker equivalence 
$ 
    X \mOnes{n} 
    = 
    (\mIdentity{m} \kron \mOnes{n}^T) x
$,
where $x \in \Real^{mn}$ is a length-$mn$ vector formed from $X \in \mnReal$ by setting 
$ 
    \idx{x}{(k-1)n + \ell} 
    = \ijdx{X}{k}{\ell}
$ 
.
Also be the same equivalence we obtain column-sums of $X \in \mnReal$ 
$
    X^T \mOnes{m}
    = 
    (\mOnes{m} \kron \mIdentity{n}) x
$
from the same equivalent $x \in \Real^{mn}$;
in \texttt{numpy} terminology\footnote{
\texttt{https://numpy.org/doc/stable/reference/generated/numpy.ravel.html}
}~one writes $x = \ravel{X}$.
We thus conclude that Euclidean projection
$
    \hat{X} = \proj{\constraintThree}{X}
$
for any $X \in \mnReal$, can be equivalently obtained by solving the following optimization over $\Real^{mn}$:
\begin{equation}
    \hat{x} = 
    \arg
    \min_{
        y \in \Real^{mn}
    }~
    \frac{1}{2}
    \norm{y-x}{2}^2,~~
    \mbox{s.t.}~~~
    (I_m \kron 1_n^T) y = a
    ,~~~
    (1_m^T \kron I_n) y = b
    \label{rowcolsum:eqn:euclidean-projector}
\end{equation}
and then setting $\hat{X} = \unravel{\hat{x}}$, where $\unravel{\cdot}$ is the operational inverse of numpy function $\ravel{\cdot}$.

Converting the problem to vectors~\eqref{rowcolsum:eqn:euclidean-projector} results in  KKT conditions with $(m+n)$ (Lagrangian) variables $\alpha \in \Real^m$ and $\beta \in \Real^n$ in a system of linear equations:
\begin{eqnarray}
    &\matrix{cc}{
        \mIdentity{mn}    & \Aeq^T \\
        \Aeq                  & 
    }
    \vector{y \\ \alpha \\ \beta}
    = 
    \vector{x \\ a \\ b} 
    \label{rowcolsum:eqn:KKT}
\end{eqnarray}
where here $\Aeq$ is an $(m+n) \times mn$ matrix with first $m$ rows equal $I_m \kron 1_n^T$ and last $n$ rows equal $1_m^T \kron I_n$.
That is to compute $\hat{X} = \proj{\constraintThree}{X}$, one may simply solve~\eqref{rowcolsum:eqn:KKT} for $y$, given $x= \ravel{X}, a$ and $b$.
However we have to pay attention that $\Aeq$ in~\eqref{rowcolsum:eqn:KKT} is \emph{not full-rank}.
To compute the exact rank of $\Aeq$, we can make use of  Remark~3.1 from \cite{peyre_computational_2020}; specifically, it is $m + n - 1$.
This implies that the \emph{left-inverse} of $\Aeq$ is not equal to $\inv{(\Aeq^T\Aeq)}\Aeq^T$, and instead requires the \emph{pseudo-inverse} of the matrix\footnote{
    This can be verified to satisfy the conditions in p.~257,~{Golub, Gene H. and Van Loan, Charles F.},~\textit{Matrix Computations},~{Johns Hopkins University Press}, Third Edition, 1996.
} $\Aeq \in \Real^{(m+n) \times mn}$ in~\eqref{rowcolsum:eqn:KKT}:
\[
    \pinv{\Aeq} = \matrix{cc}{
        \frac{1}{n} \left(
            \mIdentity{m} \kron \mOnes{n} - \frac{1}{m + n}
        \right) & 
        \frac{1}{m} \left(
            \mOnes{m} \kron \mIdentity{n} - \frac{1}{m + n}
        \right)
    }.
\]

Recall $\meanproj{k}$ is the projection $\meanproj{k} = \frac{1}{k} \mOnes{k} \mOnes{k}^T$, and matrices $M := \mIdentity{m} - \frac{m}{m+n} \meanproj{m}$ and $N := \mIdentity{n} - \frac{n}{m+n} \meanproj{n}$. 
Then~\eqref{rowcolsum:eqn:main-claim} will prove Prop.~\ref{prop:euclidean-projector-sums} 
,
restated here in the Kronecker equivalent form~\eqref{rowcolsum:eqn:euclidean-projector} corresponding to numpy function $x= \ravel{X}$.

\begin{lemma}
    \label{rowcolsum:lem:invert}
    Assume $a \in \Real^m $ and $b \in \Real^n$ have same means,~\ie,~$ a^T \mOnes{n} = b^T \mOnes{n}$. 
    Then for any $x\in \mnReal$, the solution of~\eqref{rowcolsum:eqn:KKT} is given by the the following expression for $y \in \Real^{mn}$ for some
    $ \alpha \in \Real^m$ 
    and $\beta \in \Real^n$, where
    \begin{equation}
        y =  \pinv{\Aeq} \vector{ a \\ b }
                +  \left( 
                    \mIdentity{mn} - \Aeq^T 
                    \pinv{
                        \left(\Aeq^T \right)
                    }
                    \right) x
        \label{rowcolsum:eqn:main-claim} 
    \end{equation}
    where $\pinv{A}$ is the \emph{pseudo-inverse} of $\Aeq \in \Real^{(m+n) \times mn}$ in~\eqref{rowcolsum:eqn:KKT}.
\end{lemma}

\begin{proof}[Proof of Lemma~\ref{rowcolsum:lem:invert}]

Suppose $\vector{a \\b} $ lies in $\mspan{\Aeq}$, then
our strategy is to show $y\in \Real^{mn}$ in~\eqref{rowcolsum:eqn:main-claim} satisfies both top (\ie, first $m$ rows) and bottom (\ie, last $n$ rows) of~\eqref{rowcolsum:eqn:KKT}, in two steps.
From the Moore-Penrose property $\Aeq\pinv{\Aeq} \Aeq = \Aeq$ of the pseudo-inverse, we conclude that $\pinv{\Aeq}$ has a \emph{left-inverse} property over $\mspan{\Aeq}$.
Thus, we can construct some
$
    y = \pinv{\Aeq} \vector{a \\ b} + \nu
$
with an unrestricted choice for $\nu \in \nullspace{\Aeq}$, 
and $y$ will satisfy the bottom of~\eqref{rowcolsum:eqn:KKT}.
Next we want to show our construction for $y$ also satisfies the top of~\eqref{rowcolsum:eqn:KKT}.
To do this, we show there exists some $\alpha \in \Real^m, \beta \in \Real^n$ that satisfies
$
    \Aeq^T \vector{\alpha \\ \beta} 
    = 
    x - \nu - \pinv{\Aeq} \vector{a \\ b},
$
for some specific $\nu \in \nullspace{\Aeq}$.
Indeed for any $x\in \Real^{mn}$, there exists some $\nu \in \nullspace{\Aeq}$ such that $x - \nu$ equals the projection of $x$ onto $\mspan{\Aeq^T}$,~\ie,~there exists $\nu$ such that
$
    x - \nu = \Aeq^T \pinv{(\Aeq^T)} x
$
; to check write $\nu = \left(\mIdentity{mn} - \Aeq^T \pinv{(\Aeq^T)}\right) x$ and observe that $\nu$ is then the projection of any $x\in\Real^{mn}$ onto $(\mspan{\Aeq^T})^\perp$, and observe that $(\mspan{\Aeq^T})^\perp = \nullspace{\Aeq}$, agreeing with our assumption $\nu \in \nullspace{\Aeq}$. 
Thus, we have shown that $y = \pinv{\Aeq} \vector{a \\ b} + \nu$ with $\nu = \left(\mIdentity{mn} - \Aeq^T \pinv{(\Aeq^T)}\right) x$ satisfies the optimality conditions~\eqref{rowcolsum:eqn:KKT}, and this form of $y$ is exactly that of~\eqref{rowcolsum:eqn:main-claim}. 

To conclude we need to prove the starting assumption $\vector{a \\ b} \in \mspan{A}$. Because we assumed $1^T a = 1^T b$ we have $\vector{1_m \\ -1_n}$ to be $\perp$ to $\vector{a \\ b}$. On the other hand $\vector{1_m \\ -1_n}$ is in fact in $(\mspan{A})^\perp$. Therefore $\vector{a \\ b} \in \mspan{A}$.
\end{proof}

\paragraph{Restatement of Proposition~\ref{prop:euclidean-projector-sums}} For any $x \in \Real^{mn}$, the Euclidean projection $\hat{x}$ solving~\eqref{rowcolsum:eqn:euclidean-projector} satisfies $\hat{x} = y_1 + (x - y_2)$, where 
$
    y_1 = \frac{1}{n} [M \kron \mOnes{n}] a 
    + \frac{1}{m} [\mOnes{n} \kron N] b
$ and
$
    y_2 = [M \kron \meanproj{n}] x + 
    [\meanproj{m} \kron N] x
$
for $\meanproj{m}, \meanproj{n}, M$ and $N$ defined as before.
\begin{proof}[Proof of Proposition~\ref{prop:euclidean-projector-sums}]
    Put expression for $\pinv{\Aeq}$  into the expression for $y$ in~\eqref{rowcolsum:eqn:main-claim}, set $\hat{x} = y$, and apply basic manipulations.
\end{proof}



\subsection{Bounds}

The goal here is to extend the result in~\citet{kusner_word_2015} and obtain a lower bound (see~\eqref{eqn:optimal-transport-order-constraint-lb} in main text) on the optimal cost~\eqref{eqn:optimal-transport-order-constraint2} with order constraints.
This requires a lower bound (see~\eqref{eqn:final-bound} in main text) on the function $\nijdx{\tailfun}{i}{j}{k} (x, \submodularlb)$ that appears in~\eqref{eqn:optimal-transport-order-constraint-lb} in the main text.
Prop.~\ref{prop:tailfun-bound} derives this lower bound by decoupling the expression into ($m$ or $n$) independent minimizations that can be computed in parallel; \citet{kusner_word_2015} takes the same approach for the simpler problem~\eqref{eqn:optimal-transport}, but here a global constraint (due to the order constraints) still has to be dealt with (see equation below~\eqref{eqn:optimal-transport-order-constraint-lb} in main text).

Here the independent minimizations are related to 
$\knapsack{
    \{
        \idx{\boundcoef}{k}
    \}_{k \in \iset{n}}, u, \alpha
}$, see~\eqref{eqn:packing} in the main text,
where $u$ and $\alpha$ represent the per item $\idx{x}{k}$ capacity and total budget, respectively. 
We first present 
Lemmatas~\ref{bounds:lem:knapsack-other-row}--\ref{bounds:lem:knapsack-current-row} that show the quantities $\tailfunlbRow{ijk}{x,D,a}$ and $\tailfunlbCol{ijk}{x,D,b}$ (in the lower bound~\eqref{eqn:final-bound}, main text) 
are convex and piecewise-linear.
That is the lower bound in Proposition~\ref{prop:tailfun-bound} can be efficiently evaluated   by bisection; it only requires an upfront sorting complexity of $\bigO{n\log n}$ when parallelized.

\begin{lemma}
    \label{bounds:lem:knapsack-other-row}
    The solution to $\knapsack{\{\boundcoef\}_{k \in \iset{n}}, u, \alpha}$ for $0 \leq \alpha/n < u \leq \alpha \leq 1$ is given by 
    \begin{equation}
        \knapsack{\{\boundcoef\}_{k \in \iset{n}}, u, \alpha}  = 
            u \left[
                \sum_{k = 1}^\ell 
                \left(
                    \odidx{\boundcoef}{k}
                    - 
                    \odidx{\boundcoef}{\ell+1}
                \right)
            \right]
            + \alpha  
            \odidx{\boundcoef}{\ell+1}
        \label{bounds:eqn:knapsack-other-row}
    \end{equation}
    where $\ell = \floor{\alpha/u} < n$,
    and $\odidx{\boundcoef}{k}$ is the $k$-th coefficient in increasing order 
    $
        \odidx{\boundcoef}{1} \leq \odidx{\boundcoef}{2} \leq \cdots \leq \odidx{\boundcoef}{n}
    $. 
    For fixed $\alpha$, it is 
    convex piecewise-linear and monotone non-increasing in $u$, with inflection points at $u = \alpha/i$ for $i = n-1, n-2, \cdots, 1$.
\end{lemma}

\begin{proof}
    The solution (of~\eqref{eqn:packing}, main text) is by greedy prioritization of lower costs in the ordering $\odidx{\boundcoef}{k}$, hence:
    \begin{equation}
        \knapsack{\{\phi_k\}_{k \in \iset{n}}, u, \alpha}  = 
        \sum_{k = 1}^\ell 
        \odidx{\boundcoef}{k}
        u + (\alpha - \ell \cdot u )  
        \odidx{\boundcoef}{\ell+1}
        \label{bounds:eqn:1-raw-solution}
    \end{equation}
    where $\ell = \floor{\alpha/u} < n$, and clearly
    $\alpha/n \leq u \leq \alpha$ is required to have a at least one feasible solution.
    Gathering coefficients of $u$ we arrive at~\eqref{bounds:eqn:knapsack-other-row} from~\eqref{bounds:eqn:1-raw-solution}, 
    and from~\eqref{bounds:eqn:knapsack-other-row} it is clearly piecewise-linear with inflection points at values of $u$ that cause $\ell = \floor{\alpha/u}$ to change value; there are $n-1$ of them at values 
    $u = \alpha/i$ for $i = n-1, n-2, \cdots, 1$ making a total of $n$ intervals
    $( \alpha/(i+1), \alpha/i ]$.
    Pick some $i < n$, fix some $0 \leq \alpha \leq 1$, and consider
    $u$ in the interval 
    $( \alpha/(i+1), \alpha/i ]$.
    For our choice of $i$ we conclude the sequence of inequalities
    $
        i + 1 
        = \floor{\frac{\alpha}{\alpha/(i+1)}}
        > 
        \floor{\frac{\alpha}{u}}
        = \ell 
        \geq \floor{\frac{\alpha}{\alpha/i}}
        = i
    $
    and so we conclude $\ell$ in~\eqref{bounds:eqn:1-raw-solution} will equal $\ell=i$ for any $u \in ( \alpha/(i+1), \alpha/i ]$.
    The function~\eqref{bounds:eqn:knapsack-other-row} is piecewise-linear because evaluating the left-limit $\alpha / (i+1) \leftarrow u$ of interval
    $( \alpha/(i+1), \alpha/i ]$ (when $\ell=i$), equals the value evaluated at the rightmost point of the neighboring interval $(\alpha/(i+2), \alpha/(i+1) ]$ (when $\ell=i+1$):
    \[
        \knapsack{\{\phi_k\}_{k \in \iset{n}}, \frac{\alpha}{i+1}, \alpha}  
        = 
        \lim_{
            \frac{\alpha}{i+1}
            \leftarrow
            u 
        } \knapsack{\{\phi_k\}_{k \in \iset{n}}, u, \alpha}  
        = 
        \frac{\alpha}{i+1} 
            \sum_{k = 1}^{i+1} 
            \odidx{\boundcoef}{k}.
    \]
    From~\eqref{bounds:eqn:1-raw-solution} it the gradient is non-positive everywhere due to the non-positivity of the summands 
    $
        \odidx{\boundcoef}{k}
        - 
        \odidx{\boundcoef}{i+1}
        \leq 0
    $
    since $k \leq \ell$, implying~\eqref{bounds:eqn:1-raw-solution} is monotone non-increasing in $u$ everywhere.
    A monotone non-increasing function is convex if its gradient is non-decreasing. 
    Indeed, if we subtract the gradient of the $i$-th interval (where $\ell=i$), to the from that of the $(i+1)$-th interval to the right (where $\ell=i+1$):
    \[
        \left[
            \sum_{k = 1}^{i-1}
            \left(
                \odidx{\boundcoef}{k}
                - 
                \odidx{\boundcoef}{i}
            \right)
        \right] 
        - 
        \left[
            \sum_{k = 1}^i
            \left(
                \odidx{\boundcoef}{k}
                - 
                \odidx{\boundcoef}{i+1}
            \right)
        \right] 
        = i \cdot \left(
                \odidx{\boundcoef}{i+1}
                - 
                \odidx{\boundcoef}{i}
            \right)
        \geq 0.
    \]
    We have thus proved that~\eqref{bounds:eqn:knapsack-other-row} is piecewise-linear, monotone non-increasing and convex.
\end{proof}

\begin{lemma}
    \label{bounds:lem:knapsack-current-row}
    The solution to 
    $\knapsack{
        \{
            \idx{\boundcoef}{k}
        \}_{k \in \iset{n-1}}, u, \alpha-u}$ 
    for $0 \leq \alpha/n \leq x \leq \alpha \leq 1$ is:
    \begin{equation}
        \knapsack{\{
            \idx{\boundcoef}{k}
        \}_{k \in \iset{n-1}}, u, \alpha - u}  = 
            u \left[
                - \odidx{\boundcoef}{\ell}
                + \sum_{k = 1}^{\ell-1} 
                \left(
                    \odidx{\boundcoef}{k}
                    - \odidx{\boundcoef}{\ell}
                \right)
            \right]
            + \alpha  
              \odidx{\boundcoef}{\ell}
        \label{bounds:eqn:knapsack-current-row}
    \end{equation}
    where $\ell = \floor{ \alpha/u}$
    and $\odidx{\boundcoef}{k}$ is the $k$-th coefficient in increasing order 
    $
        \odidx{\boundcoef}{1} \leq \odidx{\boundcoef}{2} \leq \cdots \leq \odidx{\boundcoef}{n}
    $. 
    For a for a fixed $\alpha$, it is convex piecewise-linear and non-increasing in $u$ with the inflection points $u = \alpha/i$ for $i = n-1, n-2, \cdots, 1$.
\end{lemma}

\begin{proof}
    The proof is identical to that of  Lemma~\ref{bounds:lem:knapsack-other-row}; we outline the key differences.
    The expression~\eqref{bounds:eqn:knapsack-current-row} is derived similarly as~\eqref{bounds:eqn:knapsack-other-row} by the same greedy strategy
    \[
        \knapsack{
            \{
                \idx{\boundcoef}{k}
            \}_{k \in \iset{n-1}}, u, \alpha - u}
        = 
        \sum_{k=1}^{\ell-1} 
        \odidx{\boundcoef}{k} u
        + (\alpha - \ell u) 
        \odidx{\boundcoef}{\ell}.
    \]
    To show the monotonic non-increasing behavior and convexity proceed in the similar manner.
    The $n-1$ inflection points for $i = n-1, n-2, \cdots 1$ are the same, thus so are the piecewise intervals;
    evaluating at the left-limit $\alpha / (i+1) \leftarrow u$ of the interval $( \alpha/(i+1), \alpha/i ]$ equals
    $
        \frac{\alpha}{i+1} \sum_{k=1}^i 
        \odidx{\boundcoef}{k} 
    $,
    as does the rightmost point $u=\alpha/(i+1)$ of the interval $( \alpha/(i+2), \alpha/(i+1) ]$.
    The gradient is indeed non-decreasing, again as seen by subtracting the gradient in the $i$-th
    $( \alpha/(i+1), \alpha/i ]$ from that of $(i-1)$-th interval which we obtain
    $
        i \cdot \left(
            \odidx{\boundcoef}{i}
            - \odidx{\boundcoef}{i-1}
        \right) 
        \geq 0
    $ for any choice of $i < n$.
\end{proof}

\paragraph{Complexity of computing lower bounds in Proposition~\ref{prop:tailfun-bound}:} 
The lower bound (see~\eqref{eqn:final-bound}, main text) suggests to evaluate $L_{1ijk}(x,\submodularlb,a)$ and $L_{2ijk}(x,\submodularlb,b)$ at multiple values of  $\alpha \leq x \leq \beta $, or
equivalently, evaluate $\knapmu$ and $\knapnu$ (see~\eqref{eqn:function-mu-and-nu}, main text) over the same.
But  $\knapmu$ and $\knapnu$ are defined using
$
    \knapsack{\{\boundcoef\}_{i \in \iset{n}}, x, a_k} 
$ and
$
    \knapsack{\{\boundcoef\}_{i \in \iset{n-1}}, x, a_i-x} 
$,
and by Lemmata~\ref{bounds:lem:knapsack-other-row}--\ref{bounds:lem:knapsack-current-row} they are
piecewise-linear $x$ with gradients and coefficients determined upfront by an $\bigO{n \log(n)}$ sort on the coefficients $\idx{\boundcoef}{k}$.
In other words, the cost of $\bigO{n \log(n)}$ is only one-time, and once paid,  the function~\eqref{bounds:eqn:knapsack-other-row} can be evaluated for multiple points of $x$ in $\bigO{1}$ time.
Finally, each $\knapmu,\knapnu$ summand in the expressions for 
$L_{1ijk}(x,\submodularlb,a)$ and $L_{2ijk}(x,\submodularlb,b)$ can be computed independently in parallel.

\paragraph*{Restatement of Proposition~\ref{prop:tailfun-bound}}

   Let
    $
        \alpha_1 = \max_{i=1}^m \frac{a_i}{n},
        \beta_1 = \max_{i=1}^m a_i,
    $, and
    $
        \alpha_2 = \max_{j=1}^n \frac{b_j}{m},
        \beta_2 = \max_{j=1}^n b_j.
    $
    Then for any $\nij{i}{j}{k}$ set that defines the order constrained OT (see~\eqref{eqn:optimal-transport-order-constraint2} in main text) where $\ni{i}{k}=i_1,i_2,\cdots, i_k$ (and $\ni{j}{k}$) do not repeat row (or column) indices, the minimum $\nijdx{\tailfun}{i}{j}{k} (x, \submodularlb)$ (see~\eqref{eqn:optimal-transport-order-constraint-lb} in main text)
    is lower-bounded by
    \begin{eqnarray}
        \nijdx{\tailfun}{i}{j}{k} (x, \submodularlb)
        \!\!\! &\geq& \!\!\!
        \left\{
            \begin{array}{cc}
                \tailfunlbRow{ijk}{x,\submodularlb,a} & \mbox{for}~~ \alpha_1 \leq x \leq \beta_1
                \\
                \tailfunlbCol{ijk}{x,\submodularlb,b} & \mbox{for}~~  \alpha_2 \leq x \leq \beta_2
            \end{array}
        \right.
        \nonumber
    \end{eqnarray}
    where $\tailfunlbRow{ijk}{x,\submodularlb,a}$ and $\tailfunlbCol{ijk}{x,\submodularlb,b}$ resp. equal
    $
        \sum_{\ell \in \iset{k}}
        \knapnu
        \left(x, 
            \{
                \idx{\boundcoef}{\ij{i_\ell}{q}}
            \}_{q \in \iset{n} \setminus \{j_\ell\}},
            a_{i_\ell}
        \right)
        + 
        \sum_{p \notin \ni{i}{k}} 
        \knapmu
        \left( 
            x, 
            \{
                \idx{\boundcoef}{\ij{p}{q}}
            \}_{q \in \iset{n}}, 
            a_p 
        \right)
    $
    and
    $
        \sum_{\ell \in \iset{k}}
        \knapnu
        \left(
            x, 
            \{
                \idx{\boundcoef}{\ij{p}{j_\ell}}
            \}_{p  \in \iset{m} \setminus \{i_\ell\}}, 
            b_{j_\ell}
        \right)
        %
        +
        \sum_{q \notin \ni{j}{k}} 
        \knapmu
        \left(
            x, \{
                \idx{\boundcoef}{\ij{p}{q}}
            \}_{p \in \iset{m}}, 
            b_q
        \right).
    $

\begin{proof}[Proof of Proposition~\ref{prop:tailfun-bound}]
    For brevity we only prove the top bound that exists for $x$ in the range $\alpha_1 \leq x \leq \beta_1$ (notated as $\tailfunlbRow{ijk}{x,\submodularlb,a}$ in~\eqref{eqn:tailfun-bound} of main text); the other $ \tailfunlbCol{ijk}{x,\submodularlb,b}$ will follow similarly and is omitted.

    The bound (see \eqref{eqn:tailfun-bound}, main text) is obtained by relaxing the constraint set (a la~\eqref{eqn:optimal-transport-order-constraint-tailfunction}, main text):
    \begin{eqnarray*}
        &&\left\{
            \tplan \in \mnReal:
            \tplan \in \tpoly,~
            \tplan \in \nijdx{\order}{i}{j}{k},~
            \ijdx{\tplan}{i_1}{j_1} = \cdots = \ijdx{\tplan}{i_k}{j_k} = x
        \right\} 
        \\
        &\subseteq&
        \bigcap_{p \in \iset{m}}
        \left\{
            \tplan \in \mnReal_+:
            \sum_{q \in \iset{n}} 
            \ijdx{\tplan}{p}{q}
            = 
            a_q,~
            \ijdx{\tplan}{i_1}{j_1} = x, \cdots,
            \ijdx{\tplan}{i_k}{j_k} = x,~
            \max_{q\in\iset{n}} \ijdx{\tplan}{p}{q} \leq x
        \right\}.
    \end{eqnarray*}
    The $p$-th set on the RHS only involves coefficients $\ijdx{\tplan}{p}{q}$ found in the $k$-th row;
    after relaxation and taking into account the linear form of $\nijdx{\tailfun}{i}{j}{k} (x, \submodularlb)$,  
    and the fact that for each $\ell$-th row $i_\ell$, the $j_\ell$-th column does not repeat in $\ni{j}{k}$,
    we obtain $m$ independent minimizations of the following form:
    \begin{eqnarray*}
        \begin{array}{lll}
            \min 
            \sum_{q \in \iset{n}} 
            \ijdx{\boundcoef}{p}{q}
            \ijdx{\tplan}{p}{q}
            &
            \mbox{s.t.}~~~~~~
            \sum_{q \in \iset{n}} 
            \ijdx{\tplan}{p}{q}
            = 
            a_p,
            ~~\mbox{and}~~
            0 \leq \ijdx{\tplan}{p}{q} \leq x,
            &
            \mbox{for}~p \notin \ni{i}{k}
            \\
            \min 
            \sum_{q \in \iset{n} \setminus \{j_\ell\}} 
            \ijdx{\boundcoef}{i_\ell}{q}
            \ijdx{\tplan}{i_\ell}{q}
            &
            \mbox{s.t.}~~~~~~
            \sum_{q \in \iset{n} \setminus \{j_\ell\}} 
            \ijdx{\tplan}{i_\ell}{q}
            = 
            a_{i_\ell} - x,
            ~~\mbox{and}~~
            0 \leq \ijdx{\tplan}{i_\ell}{q} \leq x
            ,
            &
            \mbox{for}~\ell \in \iset{k}
        \end{array}
    \end{eqnarray*}
    where the sum of the $m$ optimal costs of the $m$ minimizations, lower bound the quantity
     $\nijdx{\tailfun}{i}{j}{k} (x, \submodularlb)$.
     For any $p\notin \ni{i}{k}$, the $p$-th minimization has optimal cost 
    $
        \knapmu
        \left( 
            x, 
            \{
                \idx{\boundcoef}{\ij{p}{q}}
            \}_{q \in \iset{n}}, 
            a_p 
        \right)
        =
        \knapsack{
            \{
                \idx{\boundcoef}{\ij{p}{q}}
            \}_{q \in \iset{n}}, 
            x,
            a_p
        } 
    $, and similarly for $\ell \in \iset{k}$, the $i_{\ell}$-th row  has optimal cost 
    $
        \knapnu
        \left(x, 
            \{
                \idx{\boundcoef}{\ij{p}{q}}
            \}_{q \in \iset{n} \setminus \{j_\ell\}}, 
            a_{i_\ell} 
        \right)
        =
        \knapsack{
            \{
                \idx{\boundcoef}{\ij{p}{q}}
            \}_{q \in \iset{n} \setminus \{j_\ell\}}, 
            x,
            a_{i_\ell} - x
        }  
    $, see~\eqref{eqn:packing} of main text.
    Lemmata~\ref{bounds:lem:knapsack-other-row}--~\ref{bounds:lem:knapsack-current-row} applied to the \texttt{PACKING} problems above, obtains that the minimization for the $p$-th rows is only feasible for $x \geq a_p/ n$, and has the same optimal value for $a_p \leq x \leq 1$;
    therefore the global constraint across all minimizations is obtained as $\alpha_1 = \max_{p \in \iset{m}} a_p /n$ and $\beta_1 = \max_{p \in \iset{n}} a_p$.

\end{proof}

\fi

\end{document}